%% file: article.tex
\documentclass{article}

\usepackage{arxiv}
\usepackage{amssymb}
\usepackage{amsmath}
\usepackage[numbers, sort&compress]{natbib}
\usepackage{lineno}
\usepackage{hyperref}
\usepackage{blindtext}
\usepackage{csquotes}
\usepackage{enumitem,amssymb}
\usepackage{bm}
\usepackage{pifont}
\usepackage{amsmath}
\usepackage{graphicx}
\usepackage{subfigure}
\usepackage{placeins}
\usepackage{algpseudocode}
\usepackage{algorithm}
\usepackage{rotating}
\usepackage{booktabs}
\usepackage{subcaption}
\usepackage{tikz}
\usepackage{adjustbox}
\usepackage{booktabs}
\usepackage{enumitem}
\usepackage[format=plain]{caption}
\usepackage{tcolorbox}
\usepackage{lipsum}
\tcbuselibrary{skins,breakable}
\usetikzlibrary{shadings,shadows}
\usepackage[parfill]{parskip}
\usepackage{stackengine}
\usepackage{censor}
\usepackage[toc,page]{appendix}
\usepackage{tabularx}
\usepackage{amsthm}
\usepackage{dsfont}
\usepackage[onehalfspacing]{setspace}

\begingroup
\makeatletter
\@for\theoremstyle:=definition,remark,plain\do{%
	\expandafter\g@addto@macro\csname th@\theoremstyle\endcsname{%
		\addtolength\thm@preskip\parskip
	}%
}
\endgroup

{\endtcolorbox}

\usetikzlibrary{bayesnet}
\usetikzlibrary{arrows}
\usetikzlibrary{shapes,arrows,calc}
\usetikzlibrary {shapes.geometric}

\pgfdeclarelayer{background}
\pgfdeclarelayer{foreground}
\pgfsetlayers{background,main,foreground}

\usepackage[T1]{fontenc}    % use 8-bit T1 fonts
\usepackage{url}            % simple URL typesetting
\usepackage{booktabs}       % professional-quality tables
\usepackage{amsfonts}       % blackboard math symbols
\usepackage{nicefrac}       % compact symbols for 1/2, etc.
\usepackage{microtype}      % microtypography

\newlist{todolist}{itemize}{2}
\setlist[todolist]{label=$\square$}
\newcommand{\cmark}{\ding{51}}%
\newcommand{\xmark}{\ding{55}}%

\definecolor{ibmorange}{HTML}{FE6100}
\definecolor{ibmpurple}{HTML}{785EF0}
\definecolor{ibmblue}{HTML}{648FFF}
\definecolor{ibmpink}{HTML}{DC267F}
\definecolor{tolsand}{HTML}{DDCC77}
\definecolor{ibmaqua}{HTML}{31C9B0}

\newtheorem{remark}{Remark}

\newtheorem{lemma}{Lemma}
\newtheorem{proposition}{Proposition}
\newtheorem{corollary}{Corollary}

\newtheorem{assumptions}{Assumption}

\title{Explainable Condition Monitoring via Probabilistic Anomaly Detection Applied to Helicopter Transmissions}

\date{}

\author{
  Aurelio Raffa Ugolini \\
  DEIB, Politecnico di Milano \\
  Via Giuseppe Ponzio 34, Milano, 20133 (Italy) \\
  \texttt{aurelio.raffa@polimi.it} \\
    \And
  Jessica Leoni \\
  DEIB, Politecnico di Milano \\
  Via Giuseppe Ponzio 34, Milano, 20133 (Italy) \\
  \And
  Valentina Breschi \\
  Dep. of Electrical Engineering, TU\textbackslash e \\
  Groene Loper 19, Eindhoven, 5600 MB (The Netherlands) \\
  \And
  Damiano Paniccia \\
  Leonardo Labs, Leonardo S.p.A. \\
  Piazza Monte Grappa 4, Roma, 00195 (Italy) \\
  \And
  Francesco Aldo Tucci \\
  Leonardo Labs, Leonardo S.p.A. \\
  Piazza Monte Grappa 4, Roma, 00195 (Italy) \\
  \And
  Luigi Capone \\
  Leonardo Labs, Leonardo S.p.A. \\
  Piazza Monte Grappa 4, Roma, 00195 (Italy) \\
  \And
  Mara Tanelli \\
  DEIB, Politecnico di Milano \\
  Via Giuseppe Ponzio 34, Milano, 20133 (Italy) \\
}

\begin{document}
\StopCensoring
\maketitle

%% Abstract
\begin{abstract}
We present a novel Explainable methodology for Condition Monitoring, relying on healthy data only.
Since faults are rare events, we propose to focus on learning the probability distribution of healthy observations only, and detect Anomalies at runtime.
This objective is achieved via the definition of probabilistic measures
of deviation from nominality, which allow to detect and anticipate faults.
The Bayesian perspective underpinning our approach allows us to perform Uncertainty Quantification to inform decisions.
At the same time, we provide descriptive tools to enhance the interpretability of the results, supporting the deployment of the proposed strategy also in safety-critical applications.
The methodology is validated experimentally on two use cases: a publicly available benchmark for Predictive Maintenance, and a real-world Helicopter Transmission dataset collected over multiple years.
In both applications, the method achieves competitive detection performance with respect to state-of-the-art anomaly detection methods.
\end{abstract}

%% Keywords

\input{sections/01_introduction}

\input{sections/03_cocoafuse}
\input{sections/04_anomaly_score}
\input{sections/05_model_evaluation}
\input{sections/06_experiments}

\section{Conclusions}\label{sec:conclusions}
	In this work, we have introduced an
    explainable Anomaly Detection technique,
    suitable for Condition Monitoring,
    and showcased its application
    to the field of Transmission Vibration Monitoring for Helicopters.
	Its key advantages lie in
    the robustness of its
    probabilistic perspective, which does not require faults to train the models.
    It is
    supported by a principled approach to Uncertainty Quantification,
    as well as its
    interpretability by design, enabled by the use of simple submodels appropriately fused throughout the input domain.
	Though experimental validation, we have demonstrated that using
    a model for the conditional response distribution
    is effective in anticipating faults,
    as proven by the competitiveness with established baselines (see~\citet{2024leoniExplainableDatadrivenModeling}).
	Still, many avenues for research remain open.
    For example, modeling the time dependence of the HIs' conditional distribution could account for the progressive wear of components,
    further distinguishing it from sudden failures.
	Also, more flexible distributions for the experts, e.g., in some heavy-tailed family, could more effectively capture the nature of the responses.
	Finally, the MCMC framework adopted here to fit the model is computationally demanding and does not scale well with data.
	We intend to investigate all of these aspects over our next works, with hopes it will foster exchanges between domain experts and analysts working on CM.

\bibliography{references.bib}

\newpage
\appendix
\input{sections/99_appendices}

\end{document}

%% file: sections/01_introduction.tex
\section{Introduction}\label{sec:introduction}
The need for probing the degradation state (\textit{diagnostics}) of active machinery without relying on invasive testing has ignited a plethora of research directions, which broadly fall under the name of ``Condition Monitoring'' (CM).
In CM, data from sensors mounted on machines (like motors, pumps, or turbines) is analyzed to identify patterns that deviate from normal behavior.
The outcomes of these analyses are valuable diagnostic tools for Predictive Maintenance (PM) in many sectors (see, e.g.,~\citet{1996raoHandbookConditionMonitoring, 2015ambhoreToolConditionMonitoring, 2022badihiComprehensiveReviewSignalBased}).
As part of a Predictive Maintenance strategy, CM allows for maintenance to be performed based on actual equipment condition rather than predefined schedules, minimizing unplanned downtime, while extending component life, and improving overall system reliability.

CM exploits the idea that the progressive degradation of the machinery, either as a consequence of normal wear-and-tear or sudden faults, influences the characteristics of some observables (``responses'') in detectable ways.
Therefore, if a suitable characterization of the system response is obtained,
one can use it to deduce the current state of the machinery, as well as the nature of a (incipient) fault, just by observing \textit{outputs} of the system.
Generally speaking, the previous task is performed in CM by classifying the state of the monitored object as either healthy (``nominal'') or not (``faulty'').
Different solutions to this problem have been proposed, resulting in distinct kinds of approaches that can be broadly categorized as follows.
\begin{enumerate}
    \item Ad-hoc, Indicator-based strategies, which rely on quantities (often referred to as ``Health Indices'') specifically engineered to discriminate among condition types.
    Although highly explainable, these approaches are generally application and context-dependent, and require expert calibration (see~\citet{2014zimrozDiagnosticsBearingsPresence, 2021soolonwahRegressionbasedDamageDetection}).
    They are thus not flexible or easy to automate.
    \item Supervised Learning strategies, which aim to detect and rate a possible fault based on labeled healthy and faulty examples (see, e.g.,~\citet{2019stetcoMachineLearningMethods, 2020mubaraaliIntelligentFaultDiagnosis, 2023yessoufouClassificationRegressionbasedConvolutional}).
    Achieving a representative sample of each of the target classes to train a supervised classifier, however, can be difficult -- and extremely expensive -- considering that faults are generally rare (see~\citet{2022leoniNewComprehensiveMonitoring}).
    \item Single-class Learning strategies, wherein the condition classification rule is derived from healthy data only.
    These techniques aim at building boundaries around healthy data to detect faulty instances at a later time (deployment)\footnote{This task is known as
    Anomaly Detection (see~\citet{2009chandolaAnomalyDetectionSurvey}).}.
    One-Class Support Vector Machines (SVMs) (e.g.,~\citet{2021turnbullCombiningSCADAVibration,
    2022britoExplainableArtificialIntelligence}, and~\cite{2022huIntelligentAnomalyDetection}) and Autoencoders (e.g.,~\citet{2022leoniNewComprehensiveMonitoring, 2023mengAnomalyDetectionConstruction,
    2024jungAIBasedAnomalyDetection}) are popular tools for such use cases.
    For learning-based methods, nevertheless, a trade-off is usually present as more complex models tend to be more Flexible, albeit less Interpretable.%
\end{enumerate}
A schematic overview of the aforementioned methodologies can be found in Table~\ref{tab:anomaly-detection-approaches}, rated according to their General applicability to different contexts (\textbf{G}), the Interpretability of their decision-making logic (\textbf{I}), their Flexibility in modeling (\textbf{F}), and the possibility of devising them via Healthy data only (\textbf{H}).
As is clear from Table~\ref{tab:anomaly-detection-approaches}, a first contrast is immediately apparent among approaches relying on domain knowledge (1) versus data-driven insights (2-3, e.g., see~\citet{2023surucuConditionMonitoringUsing, 2024arenaConceptualFrameworkMachine}). Indeed, while the formers guarantee Interpretability, they lack the General applicability and Flexibility of the other classes of approaches. In contrast, the latter guarantee Generality and Flexibility, often at the expense of transparency.
\begin{table}[t!]
    \begin{tabularx}{\textwidth}{Xllll}
        \toprule
            Strategy&
                \textbf{G}&
                \textbf{I}&
                \textbf{F}&
                \textbf{H}\\
        \midrule
            Ad-Hoc Indicators (e.g.,~\citet{2021soolonwahRegressionbasedDamageDetection})&
                \xmark&
                \cmark&
                \xmark&
                \xmark\\
            Multi-class Learning (e.g.,~\citet{2023yessoufouClassificationRegressionbasedConvolutional})&
                \cmark&
                \xmark&
                \cmark&
                \xmark\\
            Single-class Learning (e.g.,~\citet{2024jungAIBasedAnomalyDetection})&
                \cmark&
                \xmark&
                \cmark&
                \cmark\\
        \bottomrule
    \end{tabularx}
    \caption{Condition Monitoring (CM) Strategies and their typical Advantages (\textbf{G}: General applicability; \textbf{I}: Interpretability; \textbf{F}: Flexibility; \textbf{H}: requiring a Healthy sample only).}\label{tab:anomaly-detection-approaches}
\end{table}

Irrespective of the approach taken, two modeling challenges are to be expected in CM applications.
First, explicit characterizations of the response of the system under \textit{any} possible condition are practically unattainable.
Second, consistently with a notion of ``progressiveness'' of the observable effects,
such responses are mostly overlapped when the severity of degradation is similar (\citet{2005choyVibrationMonitoringDamage, 2014zimrozDiagnosticsBearingsPresence}).
Both challenges are compounded by the fact that, in practice, only a finite set of noisy observations is available.
Due to such issues, approaches based on static thresholds or supervised classification algorithms may fail, especially in low-data regimes.
Single-class methods are thus not only \textit{convenient} for CM when little data from the faulty class(es) is available, but might become a necessity \textit{altogether} when the actual state of the machinery is not fully or sufficiently observed.

\paragraph{Our take on the problem}
To mitigate the issues related to the rarity of fault events, our approach in this work consists of detecting departures from the nominal response
through a single-class learning approach (Anomaly Detection).
Instead of postulating hard boundaries around what we assume to be the nominal condition of the machinery, however, we
use probabilistic methods to learn the statistical characteristics of healthy observations.
By carefully selecting the class of probabilistic techniques to employ,
we ensure transparency to end users.
Specifically, we adopt CoCoAFusE (\citet{2025raffaugoliniCoCoAFusEMixturesExperts}),
since it can handle both multi-modal and smooth descriptions, while explaining the effects of individual predictors on the output.
Finally, we use the aforementioned descriptions to evaluate how (in)consistent new observations are with respect to the healthy condition.
This information can then be used to determine
how closely the response matches the healthy state of the machinery.
This is an important aspect of our methodology since, by taking uncertainty into account, it provides a more general and robust decision-making strategy than setting fixed thresholds.
While we see the breadth of our methodology as very general, in light of the character of data they apply to, we will test it on the following application contexts.

\paragraph{Machine Fault Detection}
Machine Fault Detection is a specific type of data-driven CM focused on identifying deviations in equipment behavior that may indicate potential mechanical failures. Using real-time telemetry data such as current, temperature, or pressure signals, it aims at highlighting and anticipating faults.

\paragraph{Vibration Condition Monitoring (VCM) of Helicopter Transmissions}
VCM is a family of techniques within CM to assess mechanical equipment health, by specifically focusing on \textit{vibration} signals (see~\citet{2022tiboniReviewVibrationBasedCondition, 2023fengReviewVibrationbasedGear, 2023romanssiniReviewVibrationMonitoring}).
It has gained popularity in fields
such as Manufacturing (\citet{2014leiConditionMonitoringFault}), Energy and Utilities (\citet{2019wangVibrationBasedCondition}), and Civil Engineering (\citet{2016dasVibrationbasedDamageDetection}),
by virtue of its applicability to rotating machinery.
In particular, VCM plays a critical role in the maintenance and safety of Helicopter Transmission systems (see, e.g.,~\citet{2005samuelReviewVibrationbasedTechniques, 2021mauricioAdvancedSignalProcessing}), which include gearboxes, shafts, bearings, and clutches responsible for transferring power from the engine to the rotors.
Transmission systems are vital in helicopters as they provide non-redundant paths for power transfer from engines to rotors.
Therefore, a fault in even a single gear can compromise the entire helicopter, posing high costs and safety risks (\citet{2022leoniNewComprehensiveMonitoring}).
At the same time, due to their complexity, high loads, and harsh environments, transmission components are prone to wear, fatigue, and failure.
These criticalities explain the long-standing and active research focus on VCM for Helicopter Transmission systems.

\paragraph{Outline}
The rest of the article is organized as follows.
In Section~\ref{sec:problem-statement}, we introduce the addressed problem.
Section~\ref{sec:cocoafuse} summarizes our modeling choices for learning the nominal distributions of the responses.
In Section~\ref{sec:anomaly-scores}, we present the main contributions of this work, defining a set of probabilistic measures of deviation from a nominal distribution, which will be then used for Anomaly Detection.
In Section~\ref{sec:mod-eval}, we formalize the evaluation metrics adopted to select and assess our algorithm;
we also include a set of tools tailored to make the predictive machinery more transparent to end users.
Sections~\ref{sec:experiments-azure}-\ref{sec:experiments-helicopter} outline the experimental datasets used to validate the technique and expose our findings.
In detail, we present a Machine Fault Detection example in Section~\ref{sec:experiments-azure}, while a Transmission Monitoring application in Section~\ref{sec:experiments-helicopter}.
Finally, Section~\ref{sec:conclusions} contains a discussion and our conclusions.

\section{Setting and Goal}\label{sec:problem-statement}
Let $y$ (``response'') a specific Health Index (HI), i.e., one of the variables which are monitored to highlight the presence of (impending) faults or anomalies.
Meanwhile, let $\bm{x}$ (``covariates'') indicate all additional variables that are assumed to have an effect on $y$, which include other HIs as well as additional observables, e.g., pertaining to conditions external to the machinery.
Lastly, let index $t$ refer to a sampling instant for the observables, so that $\left(\bm{x}_t, y_t\right)$ denotes a covariates-response pair observed at time $t$.

If $\bm{x}_t$ were to encode the entire state of the machinery, under perfect knowledge of all physical processes acting on the machinery (e.g., mechanical, kinematic, thermodynamical, etc.), we could in principle exploit a deterministic map $\bm{x}_t \mapsto y_t$ to compure the response.
However, neither all conditions nor physical interactions are fully known.
We thus postulate the existence of a conditional probability distribution,
explaining the influence of $\bm{x}_t$ on $y_t$,
\begin{equation}\label{eq:hi-conditional-distribution}
    y_t | \bm{x}_t, \bm{\theta} {\sim} p\left(y_t | \bm{x}_t;\bm{\theta}\right),
\end{equation}
which we refer to as the ``model'' for the relationship between $\bm{x}_t$ and $y_t$.
This model
depends on a
set of unknown
parameters, which we indicate by $\bm{\theta}$.
Assuming~\eqref{eq:hi-conditional-distribution} exists, we face the problem of ``learning'' it from a sample of data, i.e. $\mathcal{D} = \left\{ \left(y_t, \bm{x}_t\right),\ t=1,\dots,N\right\}$.
Notice that~\eqref{eq:hi-conditional-distribution} alone does not specify what is the joint density of $\{y_t\}_{t=1}^N$ conditionally on $\{\bm{x}_t\}_{t=1}^N$ and $\bm{\theta}$.
To this end, we assume the following.
\begin{assumptions}\label{ass:conditional-independence}
	Consider a dataset $\mathcal{D}$, $\bm{x}_t\in\mathbb{R}^{n}$ and $y_t\in\mathbb{R}$ for $t=1,\dots,N$.
	The responses $\bm{y} = [y_1, \dots, y_N]^\top$ given the covariates $\bm{X} = \left[\bm{x}_1, \dots, \bm{x}_N\right]$ and $\bm{\theta}$ are conditionally independent, with a likelihood of the form
    \begin{equation*}
        \textstyle p(\bm{y}|\bm{X}; \bm{\theta}) = \prod_{t=1}^N p(y_t|\bm{x}_t; \bm{\theta}).
    \end{equation*}
\end{assumptions}
The previous assumption can be interpreted as a way of formally conveying that
all useful information on $y_t$ is \textit{exclusively} contained in $\bm{x}_t$ and  $\bm{\theta}$;
the remaining variability of $y_t$
is treated as a \textit{disturbance}, acting as a source of randomness.
Net of the dependencies on $\bm{x}_t$ and  $\bm{\theta}$, such randomness
should not vary with the measurement time $t$,
which justifies both
using healthy data only in learning~\eqref{eq:hi-conditional-distribution},
and using the learnt distributions to carry out the Anomaly Detection task.
Assumption~\ref{ass:conditional-independence} thus greatly simplifies the Bayesian workflow (see~\citet{2021vandeschootBayesianStatisticsModelling}), which has the goal of obtaining a \textit{posterior} distribution for the model parameters,
representing the knowledge that was gained by analyzing the collection of observations in $\mathcal{D}$.

Even when the conditional law of a target $y$ on a set of covariates $\bm{x}$ (\textit{and} model parameters $\bm{\theta}$) is available as per~\eqref{eq:hi-conditional-distribution}, it is still nontrivial to conclude that an observation does not follow the same law. Moreover,
this conclusion cannot usually be reached with absolute certainty (i.e., deterministically).
Indeed, except under very particular characteristics of the law $y|\bm{x}; \bm{\theta}$ (e.g., finite or countable support), a specific realization of $y$ can be more or less likely, but not (almost-surely) impossible.
In this work, we %would then like
aim to distinguish those realizations that are very \textit{improbable} according to our conditional model, albeit not \textit{impossible}.
To do so, we shall define a quantity that measures how unlikely an observation (or a set thereof) is according to the model.

%% file: sections/03_cocoafuse.tex
\section{Learning Conditional Densities with CoCoAFusE}\label{sec:cocoafuse}
As a first step toward devising our Bayesian Anomaly Detection strategy, we aim at learning a density for the responses $y_t$ given a corresponding set of covariates $\bm{x}_t$. To this end, we use an extension of the Mixture of Experts (MoE, see~\citet{2012yukselTwentyYearsMixture}), called \emph{Competitive/CollAborative Fusion of Experts} (CoCoAFusE), recently proposed in~\citet{2025raffaugoliniCoCoAFusEMixturesExperts}.
Belonging to the family of Conditional Dependent Mixture Models (CDMMs), known to balance smoothness and flexibility (see~\citet{2023wadeBayesianDependentMixture}), this choice guarantees the model's interpretability. Indeed, CoCoaFusE can be interpreted as a stochastic covariates-dependent selection of an appropriate sub-model, followed by a prediction from the latter, with the additional advantage of better capturing intermediate scenarios between the local descriptions
by leveraging the ``fusion'' mechanism.
Fusion enables the modeling of the response within a continuous family of interpolations between a mixture of the local sub-models (which is employed in the classical MoE), and a single density which combines the characteristic of all sub-models (the ``blend'').
The elements that characterize CoCoAFusE are:
\begin{enumerate}
    \item \textbf{The base experts $\{\mathcal{E}_i\}_{i=1}^{M}$}, namely a set of $M$ generating mechanisms (sub-models) for the target variable $y_t \in \mathbb{R}$, conditional on a set of covariates $\boldsymbol{x}_t$ and (uncertain) parameters $\boldsymbol{\theta}_i \in \mathbb{R}^{n_e}$;
    \item \textbf{The mixing gate $\mathcal{G}$}, generating a weights over the set of sub-models conditionally on $\bm{x}_t \in \mathbb{R}^{n}$ and gate-specific parameters $\bm{\theta}_{\mathcal{G}} \in \mathbb{R}^{n_{g}}$;
    \item \textbf{The behavior gate $\mathcal{B}$}, with parameters $\bm{\theta}_{\mathcal{B}}\in \mathbb{R}^{n_{b}}$, that outputs a scalar $\beta \in (0,1)$ controlling the trade-off between between blending/collaboration ($\beta\to0$) and mixing/competition ($\beta\to1$);
    \item A set of \textbf{priors} on the uncertain parameters $\bm{\theta}$ (i.e., the collection of $\bm{\theta}_1,\dots,\bm{\theta}_M, \bm{\theta}_{\mathcal{G}}, \bm{\theta}_{\mathcal{B}}$), which describe the initial level of uncertainty.
\end{enumerate}

Thanks to Assumption~\ref{ass:conditional-independence}, these elements can be used to specify the model's likelihood by only defining the conditional densities $p(y_t|\bm{x}_t; \bm{\theta})$ as
%Without loss of generality, we shall drop the index $n$ and proceed with a generic pair $(\bm{x}, y)$.
\begin{equation}\label{eq:cocoafuse-likelihood}
    \textstyle p(y_t|\bm{x}_t; \bm{\theta}) = \textstyle \sum_{i=1}^M p_i^{(\beta)}\left(y_t|\bm{x}_t; \bm{\theta}\right) p(z = i | \bm{x}_t; \bm{\theta}_{\mathcal{G}}),
\end{equation}
%In~\eqref{eq:cocoafuse-likelihood},
where $z \in \{1,\ldots,M\}$, denotes the (latent) allocation to expert $i \in \{1,\ldots,M\}$, and the associated probability $p(z=i|\bm{x}_t,\bm{\theta}_{\mathcal{G}})$, conditional to the covariates $\bm{x}_t \in \mathbb{R}^{n}$ and mixing gating parameters $\bm{\theta}_{\mathcal{G}}$.
Meanwhile, the density $p_i^{(\beta)}\left(y_t|\bm{x}_t; \bm{\theta}\right)$, is the result of the fusion mechanism applied to the base experts, % $i \in \{1,\ldots,M\}$,
according to the output $\beta$ of the behaviour gate (see~\citet[Sections 2-3]{2025raffaugoliniCoCoAFusEMixturesExperts}).
To complete our description of CoCoAFusE,
we specify our choices for the gates, experts, and prior distributions.

\paragraph{Parameterization of the gates}
A possible choice for the gate probabilities is to consider a categorical distribution, namely
\begin{equation}\label{eq:cat_distr}
    \textstyle z|\bm{x};\bm{\theta}_{\mathcal{G}}\sim\textup{Cat}(\alpha(\bm{x};\bm{\theta}_{\mathcal{G}})),
\end{equation}
where $\alpha:\mathbb{R}^{n} \rightarrow \Delta^{M}$ is a user-defined shaping function based on the covariates ($\Delta^M$ is the $M-$dimensional simplex).
In our test,%
\begin{equation}\label{eq:gate}
    \textstyle \bm{\alpha}(\bm{x};\bm{\theta}_{\mathcal{G}})=
    \textup{softmax}(\bm{A}_{\mathcal{G}}\Phi(\bm{x})),
\end{equation}%
where $\Phi:\mathbb{R}^{n} \to \{1\}\times\mathbb{R}^{n_a-1}$ is an embedding function, and $\bm{A}_{\mathcal{G}}\in\mathbb{R}^{M \times n_a}$ is a matrix-rearrangement of $\bm{\theta}_{\mathcal{G}}$.
In this way, the local explainability of predictions is preserved as a gated superimposition of \enquote{simple} decisions.
We further define $\beta$ as the output of an additional \textit{behavior} gate, with parameters $\boldsymbol{\theta}_{\mathcal{B}}\in\mathbb{R}^{n_a}$ defined, relying on the same shaping function $\Phi$:%
\begin{equation}\label{eq:behavior-gate}
    \textstyle	\beta = {\exp\left(\bm{\theta}_{\mathcal{B}}^\top \Phi(\bm{x})\right)}/\left({1 + \exp\left(\bm{\theta}_{\mathcal{B}}^\top \Phi(\bm{x})\right)}\right) = \textup{expit}\left(\bm{\theta}_{\mathcal{B}}^\top \Phi(\bm{x})\right),
\end{equation}%
\begin{remark}[Gate's priors and identifiability issues]\label{rem:mix-gate-param}
    When the mixing gate is parameterized as discussed,
    pathological behaviors might arise as the softmax function is invariant to translations of its inputs by identical offsets in each coordinate (i.e., by vectors $\bm{k} = [k, \dots, k]^\top\in\mathbb{R}^M$ with $k\in\mathbb{R}$). This issue can be addressed by setting the last row of $\bm{A}_\mathcal{G}$ to zero.
    This implies that the last expert always acts as a ``reference'' for the other inputs to the softmax.
\end{remark}
\paragraph{Affine Gaussian Base Experts} We focus on the specific case where the base densities, conditionally on the covariates $\bm{x}_t\in\mathbb{R}^{n}$, are proper Gaussians, with mean function affine in $\bm{x}_t$.
The fusion mechanism hence boils down to
\begin{equation}\label{eq:experts-difference-equation}
    p_i^{(\beta)}\left(y_t|\bm{x}_t; \bm{\theta}\right) =  \mathcal{N} \left(\theta_{i, 0}^{(\beta)} + \bm{x}_t^\top\bm{\theta}_{i, 1:n}^{(\beta)}, \theta_{i, n+1}^{(\beta)} \right),
\end{equation}
where, by the fusion mechanism (see~\citet{2025raffaugoliniCoCoAFusEMixturesExperts}):
\begin{subequations}\label{eq:fusion-parameters}
    \begin{align}
        \bm{\theta}_{i, 0:n}^{(\beta)} &= \textstyle\beta \bm{\theta}_{i, 0:n} + (1 - \beta) \left(\sum_{i \in S}\alpha_i \bm{\theta}_{i, 0:n}\right),\\
        \theta_{i, n+1}^{(\beta)} &= \textstyle\sqrt{\beta \theta_{i, n+1}^2 + (1 - \beta) \left(\sum_{i \in S}\alpha_i\theta_{i, n+1}^2\right)}
    \end{align}
\end{subequations}
with $\boldsymbol{\theta}_i \in \mathbb{R}^{n + 2}$ being the parameters of the base expert $\mathcal{E}_i$, for $i=1,\dots,M$.

\paragraph{Choice of Prior Distributions}
Experts are fully characterized parametrically once $\bm{\theta}$ and covariates $\bm{X}$ are assigned.
We choose Laplace-distributed independent priors on the coefficients of the base means' functions $\bm{\theta}_{i, 0:n}$, with the defaults for the location and scale parameters being 0 and 1, respectively.
Similar priors are chosen for the mixing and behaviour gate coefficients (except for the last row of the matrix $\bm{A}_\mathcal{G}$).
This choice is made in light of the slower decay of the tails compared to the Gaussians.
Meanwhile, we assign independent log-normal densities to the standard errors $\theta_{i, n+1}$.

%% file: sections/04_anomaly_score.tex
\section{Building Measures of Anomaly}\label{sec:anomaly-scores}
We now move on to the central part of this work,
discussing our definition of
an indicator of ``plausibility'' for observations, according to the model~\eqref{eq:hi-conditional-distribution}.
Given a (consecutive) sequence of realizations $\mathcal{P}_t^{k} = \{(y_s, \bm{x}_s)\}_{s=t - k}^{t}$,
our objective is to define a measure
$AS_{\mathcal{D}}\left(\mathcal{P}_t^{k}\right)$,
of how ``extreme'' a window of observations
$\mathcal{P}_t^{k}$
is, upon having seen a dataset of examples $\mathcal{D}$.
To do so, we first define an instrumental quantity $Q\left(\mathcal{P}_t^{k}; \bm{\theta}\right)$, fulfilling
\begin{equation}\label{eq:F-function-property}
    \mathbb{P}\left( Q\left(\mathcal{P}_t^{k}; \bm{\theta}\right) \le u | \bm{\theta}\right) \stackrel{\forall u \in (0, 1)}{=} u.
\end{equation}
Throughout the remainder of the section, we shall build such a quantity $Q\left(\mathcal{P}_t^{k}; \bm{\theta}\right)$, assuming that the data generating mechanism is precisely that of (independent) samples from $y|\bm{x}; \bm{\theta}$.
The proofs of the following results are reported in~\ref{app:proofs} for brevity.
We first introduce the following.
\begin{proposition}[The Probability Transform]\label{prop:cdf-transform}
	Let $Z$ be a continuous random variable with full support on $\mathbb{R}$ with cumulative distribution function $F_Z:\mathbb{R}\to(0, 1)$, i.e., $Z \sim F_Z$.
	Then, $U = F_Z(Z)$ has uniform distribution on $(0, 1)$.
\end{proposition}
\begin{corollary}\label{cor:cdf-transform-compact}
	Proposition~\ref{prop:cdf-transform} is easily extended to continuous random variable with full support on a compact set in $\mathbb{R}$.
\end{corollary}
Proposition~\ref{prop:cdf-transform} entails that, if $y|\bm{x}; \bm{\theta}$ with cumulative distribution function $F(\bullet |\bm{x}; \bm{\theta})$ describes a continuous random variable with full support, then
$u = F(y|\bm{x}; \bm{\theta})$ is uniformly distributed.
As will become clear in the following, extremes values for $u$ indicate very unlikely observations $y$ according to our model $y|\bm{x}; \bm{\theta}$.
However, our hypothesis is that malfunctions have a progressive onset and become more and more visible as a failure approaches.
In other words,
we aim to obtain similar results on entire sequences of observations $\mathcal{P}_t^{k}$,
without neglecting the ordering of observations.
A first step in this direction is conveyed by the following Lemma.
\begin{lemma}\label{lemma:cond-indep-of-u}
	Let $y_1 | \bm{x}_1; \bm{\theta} \sim p(\bullet | \bm{x}_1; \bm{\theta})$, $y_2 | \bm{x}_2; \bm{\theta} \sim p(\bullet | \bm{x}_2; \bm{\theta})$ be continuous random variables, with $y_1$ and $y_2$ independent given $\bm{x}_1, \bm{x}_2, \bm{\theta}$.
	Then, $u_1 = F(y_1|\bm{x}_1; \bm{\theta})$ and $u_2 = F(y_2|\bm{x}_2; \bm{\theta})$ are independent conditionally on $\bm{x}_1, \bm{x}_2, \bm{\theta}$.
\end{lemma}
\begin{remark}
	Lemma~\ref{lemma:cond-indep-of-u} can be immediately extended to any number of conditionally independent responses given the covariates and parameters.
    Notice that all of the previous assumptions are postulated by our choice for the density model (see 
    Assumption~\ref{ass:conditional-independence}),
    provided that it amounts to a valid continuous distribution on $\mathbb{R}$.
\end{remark}
Assuming knowledge of $F(\bullet|\bm{x}_s; \bm{\theta})$, which is directly encoded by our model, we can define a sequence of (conditionally) independently-distributed random variables
$\mathcal{U}_t^k = \{u_s = F(y_s|\bm{x}_s; \bm{\theta})\}_{s=t - k}^{t}$,
corresponding to each element in $\mathcal{P}_t^{k}$,
as per Proposition~\ref{prop:cdf-transform}.
Although each $u_s$ informs us of how unlikely a pair $(y_s, \bm{x}_s)$ is, we wish to convey a similar information on the entire collection $\mathcal{P}_t^{k}$.
At the same time, we want to give as much importance as possible to the most recent observations, so that the delay between the recording of an anomalous observation and the detection of a fault is minimal.
\begin{lemma}[\citet{1979barrowSplineNotationApplied}]\label{lemma:barrow79}
    If the $\mathcal{U}_t^k$ is a family of independent identically-distributed variables, each with uniform distribution on $(0, 1)$, then
    \begin{equation}\label{eq:q-quantity}
        \textstyle Q_t^k = \sum_{s=t - k}^{t} w_s u_s,
    \end{equation}
    where $W = \{w_s\}_{s=t - k}^{t}$ are non-negative and sum to 1, follows a continuous distribution with cumulative function given by:
    \begin{equation}\label{eq:weighted-sum-cdf}
        \textstyle F_W(q) = {\sum_{W^\prime \in \mathcal{W}}(-1)^{\left|W^\prime\right|}\left(q - \sum_{w\in W^\prime} w \right)_{+}^k}/\left({k!\prod_{s}w_s}\right),
    \end{equation}
    where $\mathcal{W}$ is the power set of $W$ and $\left(\bullet\right)_{+}$ is the positive part operator (maximum between argument and 0).
\end{lemma}
Owing to Lemma~\ref{lemma:barrow79}, we can set
\begin{equation}\label{eq:Q-def}
    Q\left(\mathcal{P}_t^{k}; \bm{\theta}\right) = F_W(Q_t^{k}).
\end{equation}
\begin{remark}
	We choose the expression~\eqref{eq:q-quantity} for $Q_t^k$ as it allows to weight unequally the values $\mathcal{U}_t^k$.
	The weights $W = \{w_t\}_{t=1}^T$ are chosen to be proportional to an exponential function of the lag since the most recent observation, namely $w_t \propto e^{-\lambda (T - t)}$,
	with $\lambda > 0$ a weight-decay hyperparameter, for $t = 1,\dots,T$.
\end{remark}
\begin{remark}[Reducing the Computational Complexity of~\eqref{eq:weighted-sum-cdf}]
    Although the sum morally involves $2^k$ terms at every evaluation, we can reduce the computational overhead of queries by precomputing and caching the subset sums $\mathcal{S} = \{s = \sum_{w\in W}w, W\in \mathcal{W}\}$ and subset cardinalities $\mathcal{C} = \{c = |W|, W\in \mathcal{W}\}$, which only takes milliseconds on modern machines for relatively small $k$ (e.g., lower than 15).
    Further, let $\mathcal{S} = \{s_i\}_{i=1}^{2^k}$ and $\mathcal{C} = \{c_i\}_{i=1}^{2^k}$ be indexed in such a way that the $s_i$ are non-decreasing.
    Thus, \eqref{eq:weighted-sum-cdf}~becomes
    \begin{equation}\label{eq:weighted-sum-cdf-2}
        \textstyle F_W(q) = {\sum_{i : s_i < q}(-1)^{c_i}\left(q - s_i \right)^k}/{k!\prod_{s}w_s}.
    \end{equation}
    On the average case (depending on the $w_t$),~\eqref{eq:weighted-sum-cdf-2} involves fewer terms than $2^k$.
\end{remark}
In light of Corollary~\ref{cor:cdf-transform-compact},
$Q\left(\mathcal{P}_t^{k}; \bm{\theta}\right)$ satisfies the property~\eqref{eq:F-function-property} when the data $\mathcal{P}_t^{k}$ is generated by the model with parameters $\bm{\theta}$.
Thus, $Q\left(\mathcal{P}_t^{k}; \bm{\theta}\right)$ is uniformly distributed if the $\mathcal{P}_t^{k}$ follows the model specified by $\bm{\theta}$.
We can now define:
\begin{equation}\label{eq:anomaly-score-theta}
    \textstyle AS_{\bm{\theta}}\left(\mathcal{P}_t^{k}\right)
    \textstyle =\textstyle 1 - 2\min\left\{ Q\left(\mathcal{P}_t^{k}; \bm{\theta}\right), 1 - Q\left(\mathcal{P}_t^{k}; \bm{\theta}\right) \right\}.
\end{equation}
Because of our choice for $Q\left(\mathcal{P}_t^{k}; \bm{\theta}\right)$ (see~\eqref{eq:Q-def}), if $F_W(Q_t^{k})$ approaches $0$ or $1$ (meaning that $Q_t^{k}$ lies on one of the distribution's tails), then $AS_{\bm{\theta}}\left(\mathcal{P}_t^{k}\right)$ approaches $1$.
Hence, large values of $AS_{\bm{\theta}}\left(\mathcal{P}_t^{k}\right)$ indicate extreme values for $Q_t^{k}$, as intended.
The following Lemma characterizes the \textit{frequency} at which this happens, when $\mathcal{P}_t^{k}$ follows a model with parameters $\bm{\theta}$.
\begin{lemma}\label{lemma:distribution-of-min-uniforms}
	Let $U$ be uniform on $(0, 1)$. Then, $U^\prime = 1 - 2\min\left\{U, 1-U\right\}$
	is uniformly distributed on $(0, 1)$.
\end{lemma}
\begin{corollary}\label{cor:distribution-AS}
    $AS_{\bm{\theta}}\left(\mathcal{P}_t^{k}\right)$ is a uniform on $(0, 1)$, conditionally on $\bm{\theta}$.
\end{corollary}
Corollary~\ref{cor:distribution-AS} entails that, under ``nominal'' conditions, a threshold $\tau\in (0, 1)$ on the $AS_{\bm{\theta}}\left(\mathcal{P}_t^{k}\right)$ should be exceeded on average (for repeated experiments under identical $\bm{\theta}$ and covariates $\bm{x}_t$) with probability $1-\tau$.
This legitimates using $AS_{\bm{\theta}}\left(\mathcal{P}_t^{k}\right)$ as a metric to highlight anomalous realizations of $\mathcal{P}_t^{k}$, e.g., if we see that a threshold $\tau$ is exceeded too often.
Since the model is uncertain even after fitting, we will in practice employ:
\begin{equation}\label{eq:anomaly-score-D}
    \textstyle AS_{\mathcal{D}}\left(\mathcal{P}_t^{k}\right) = \mathbb{E}_{\text{post}}\left[ AS_{\bm{\theta}}\left(\mathcal{P}_t^{k}\right) | \mathcal{D} \right],
\end{equation}
i.e., $AS_{\mathcal{D}}\left(\mathcal{P}_t^{k}\right)$, or the Anomaly Score on new data $\mathcal{P}_t^{k}$ upon having seen a dataset $\mathcal{D}$, is the posterior expectation of $AS_{\bm{\theta}}\left(\mathcal{P}_t^{k}\right)$.
Since~\eqref{eq:anomaly-score-D} cannot be computed analytically, we shall select apriori a confidence threshold $\tau$, and surrogate $AS_{\mathcal{D}}\left(\mathcal{P}_t^{k}\right)$ via a Monte Carlo based estimate $\widehat{AS_{\mathcal{D}}}\left(\mathcal{P}_t^{k}\right)$ on a parameter sample from the posterior distribution.
Then, the proxy for $AS_{\mathcal{D}}\left(\mathcal{P}_t^{k}\right)$ can be used to categorize as anomalous all (windows of) observations for which $\widehat{AS_{\mathcal{D}}}\left(\mathcal{P}_t^{k}\right) \ge \tau$.
Notice that, in~\eqref{eq:anomaly-score-D}, we select the posterior mean $\mathbb{E}_{post}[\bullet]$ for convenience, but other risk functionals could be adopted instead.

%% file: sections/05_model_evaluation.tex
\section{Selecting, Inspecting and Assessing CoCoAFusE Models}\label{sec:mod-eval}

A peculiarity of our approach is that we choose to tackle one Health Index $y_t$ at a time.
In doing so, we include all other HIs in the covariate vector $\bm{x}_t$.
This, which is not a constraint of CoCoAFusE, aids in explaining the effects of each covariate on a single Health Index, albeit sacrificing some of the information that is contained in the correlations between HIs.
The previous, however, also forces us to repeat prior elicitation and model evaluation multiple times,
which is impractical if done manually.
Hence, we propose the following autonomous hyperparameter selection procedure,
as well as qualitative tools enhancing transparency and quantitative measures to assess Fault Detection performance.

\subsection{Hyperparameter Tuning and Selection}\label{ssec:hpo}\label{sec:model-select}\label{ssec:fit-metrics}

When using CoCoAFusE, we aim at modeling the conditional target distribution given covariates. To asses this specific objective, we propose to use a few of metrics.
To evaluate the \textit{generalization} properties of the fitted model (i.e., posterior sample) on the test data, we estimate the Log Pointwise Predictive Density,
\begin{equation}\label{eq:LPPD}
    \textstyle \textup{LPPD} = \sum_{i=1}^N\ln \left(\int p({y}_i | \bm{\theta}) p(\bm{\theta} | \mathcal{D})d\theta\right),
\end{equation}
via Monte Carlo sampling. For model selection purposes, instead, we rely on the Pareto-smoothed importance sampling (PSIS) Leave-one-out (LOO) cross-validation introduced by~\citet{2017vehtariPracticalBayesianModel} evaluated on the \emph{training set} via the posterior sample. Since the PSIS-LOO is an estimator for the LPPD on unseen data, in both cases \emph{larger values} indicate \emph{better fit}.

To measure credibility coverage, i.e., how accurately the credibility bounds match with the observed responses, we instead consider
\begin{equation}\label{eq:CIC}
    \textstyle \textup{CIC}_{0.95} = \textstyle\sum_{t=1}^{N}\frac{\mathbb{P}_{p_{\bm{\lambda}}(\bm{\theta}|\mathcal{D})}\left[y_t \in \textup{CI}_{0.95}(\bm{x}_t; \bm{\theta})\right]}{N}
    \approx \textstyle\sum_{s=1}^{S}
    \frac{\mathds{1}\left\{y_t\in\textup{CI}_{0.95}\left(\bm{x}_t; \bm{\theta}^{(s)}\right)\right\}}{NS},
\end{equation}
namely the posterior 95\% Credible Interval Coverage (CIC) on a test dataset $\mathcal{D}$,
where $\textup{CI}_{0.95}(\bm{x}_t; \bm{\theta})$ is the 95\% Credible Interval evaluated at $\bm{x}_t$ given a set of model parameters $\bm{\theta}$.
Since the posterior density $p_{\bm{\lambda}}(\bm{\theta}|\mathcal{D})$ is not available in closed form, equation~\eqref{eq:CIC} approximates the $\textup{CIC}_{0.95}$ via Monte Carlo estimates from a posterior sample $\left\{\bm{\theta}^{(s)}\right\}_{s=1}^S$ of $p_{\bm{\lambda}}(\bm{\theta}|\mathcal{D})$.
Standard errors can also be attached to these estimates. Ideally, we expect the $\textup{CIC}_{0.95}$ to be \textit{very close} to the nominal value 0.95 with a small spread.

Since we intend to automate the fitting of conditional densities with CoCoAFusE, we must be able to effectively design suitable priors.
For our method to be effective in the downstream Fault Detection task (see Section~\ref{sec:anomaly-scores}),
we want to achieve a good balance between predictive capabilities, encoded by metrics such as~\eqref{eq:LPPD}, and coverage, akin to~\eqref{eq:CIC}.

The procedure for model selection should therefore achieve a suitable trade-off between the two, as discussed more in detail in~\ref{app:pareto-model-selection}.
In summary, we draw multiple configurations for the model in terms of hyperparameters, and follow the approach discussed in~\citet{2025raffaugoliniCoCoAFusEMixturesExperts} to refine them via Marginal Likelihood Maximization.
Since the Marginal Likelihood is in general non-convex,
after multiple models have been trained for different hyperparameter guesses,
we explore their Pareto frontier, ranked according to increasing fit and decreasing coverage (in this order). We accept a visited trial as new best if we can assume that it yields a better fit, at the expense of coverage, with a probability larger than a given threshold $\nu$. We empirically found that this procedure helps in selecting those trials where lower coverage is justified by substantial improvements in predictive fit.
\begin{remark}[Model selection and PSIS-LOO]
    In our implementation, we adopt PSIS-LOO on the Train set as the performance indicator guiding our choice of the model when using Algorithm~\ref{alg:pareto-model-selection} from~\ref{app:pareto-model-selection}.
\end{remark}

\subsection{Making CoCoAFusE Transparent to End Users}\label{sec:explainability}
As one of the key drivers of this work is the need for transparency in Predictive Maintenance, we structure the gates in CoCoAFusE to enable
inspection of the predictive machinery.
By setting $\Phi(\bm{x})^\top = \left[1, \bm{x}^\top\right]$,
in the parametrization, the mixing probabilities are the softmax transformation of the vector
\begin{equation}
    \textstyle \bm{v} := \bm{A}_{\mathcal{G}}\left[1, \boldsymbol{x}^\top \right]^\top = \boldsymbol{b} + \bm{A}_{0} \boldsymbol{x}.
\end{equation}
Without loss of generality, we can assume $\bm{A}_{\mathcal{G}} \in \mathbb{R}^{(M-1)\times (n + 1)}$ having one row per each expert except the last
in view of Remark~\ref{rem:mix-gate-param}.
Assuming momentarily that $\bm{A}_{0} = \left[\boldsymbol{a}_1 | \dots | \boldsymbol{a}_{M-1}\right]^\top \in \mathbb{R}^{(M-1)\times n}$ is fixed (i.e., deterministic) and full rank,
we can univocally define:
\begin{equation}\label{eq:a-star-directions}
    \boldsymbol{a}_i^\star =\textstyle \text{Proj}\left(\boldsymbol{a}_i, \text{Span}\left(\bm{A}_{0,-i}^\top\right)^\perp\right),
\end{equation}
where $\text{Span}(\bullet)$ indicates the subspace spanned by the columns of the argument, $\bm{A}_{0,-i} = \left[\boldsymbol{a}_1, \dots, \boldsymbol{a}_{i-1}, \boldsymbol{a}_{i+1}, \dots\right]^\top$, and $\text{Proj}(\bullet)$ is the projection of a vector onto a subspace.
Since $\bm{A}_{0}$ is assumedly full-rank, $\left\{\boldsymbol{a}_i^\star\right\}_{i=1}^{M-1}$ is also the basis of an $(M-1)$-dimensional subspace of $\mathbb{R}^n$.
Moving along each $\boldsymbol{a}_i^\star$ leads to increasing the score of expert $i$ without affecting the other experts'.
If $\bm{A}_{0}$ has only two rows, i.e. $M=3$, we can more easily define
\begin{equation}
    \bm{A}_{0}^\star = \left[\boldsymbol{a}_1^\star\vert\boldsymbol{a}_2^\star\right]^\top =
    \begin{bmatrix}
        1 & -\frac{\boldsymbol{a}_1\cdot \boldsymbol{a}_2}{\lVert\boldsymbol{a}_2\rVert^2}\\
        -\frac{\boldsymbol{a}_1\cdot \boldsymbol{a}_2}{\lVert\boldsymbol{a}_1\rVert^2} & 1
    \end{bmatrix}
    \left[\boldsymbol{a}_1\vert\boldsymbol{a}_2\right]^\top =: \bm{K}\ \bm{A}_{0}.
\end{equation}
Therefore, the rows of $\bm{A}_{0}^\star \in \mathbb{R}^{2 \times n}$ describe the ``effects'' of each feature on the scores of experts. Indeed, increasing by 1 feature $x_j$ without changing the others leads to an increase of $a_{1, j}^\star$ for the score of expert 1 and of $a_{2, j}^\star$ for the score of expert 2. We exploit this characteristic of $\bm{A}_{0}^\star$ to craft a visualization of the model's logic by observing the following.
Given a vector of scores $\tilde{\bm{v}}$, we can recover a fictitious point $\tilde{\bm{x}}$ such that $\bm{b} + \bm{A}_{0} \tilde{\bm{x}} = \tilde{\bm{v}}$ as
\begin{equation}
    \tilde{\bm{x}} = {{\bm{A}_{0}^\star}}^\dag (\tilde{\bm{v}} - \bm{b}) \in \text{Span}\left(\boldsymbol{a}_1^\star, \boldsymbol{a}_2^\star\right)\subset\mathbb{R}^{n}
\end{equation}
where the symbol $^\dag$ denotes the pseudo-inverse.
This consideration suggests a way to explore predictions
by representing them on the subspace spanned by the direction
$\boldsymbol{a}_i^\star$.
To allow for visualization, we focus on the case where $M\le3$ (to achieve a representation in $\mathbb{R}$ or $\mathbb{R}^2$),
but the same reasoning is applicable to larger $M$.
To do so, we first select a grid of scores from which we want to generate a ``map'' of the model. Such a map can be, e.g., a rectangle encompassing very high and low scores for the activation of each individual expert.
We then map the grid of scores onto the subspace $\text{Span}\left(\boldsymbol{a}_1^\star, \boldsymbol{a}_2^\star\right)$ to recover a fictitious dataset to be fed to the model.
\textit{Finally,} we display quantities of interest (e.g., posterior predictive mean and standard deviation, average or variance of experts activations, etc.)
back to the grid.
The benefit of this approach is that we can represent visually the effects of each input feature on both the expert selection and prediction.
Since $\bm{A}_{\mathcal{G}}$ is not a deterministic matrix, our explanations are in practice provided with respect to the \textit{posterior expectation} of the coefficient matrix.
An example of these visualizations can be found towards the end of Section~\ref{sec:experiments-helicopter}.
\begin{remark}
    When $M = 2$, we can augment the matrix ${\bm{A}_{\mathcal{G}}}$ (respectively, ${\bm{A}_{0}}$) with the coefficients from the behaviour gating and perform the same procedure, provided that the augmented matrix has full rank.
\end{remark}
\begin{remark}
    If $M > 3$, the subspace spanned by single-expert-score increasing direction has dimensionality larger than 2; for representation purposes, we can, e.g., first use SVD to approximate the posterior expected ${\bm{A}_{0}}$ with a matrix  $\tilde{{\bm{A}_{0}}}\in\mathbb{R}^{2\times n}$ and then perform the described procedure to obtain a 2D subspace in the input space $\mathbb{R}^{n}$. In this case, the new low-dimensional features do not represent single-expert activations but rather directions of increase for the most ``important'' combinations of experts, intended as those with the largest singular values of (the posterior expectation of) ${\bm{A}_{0}}$.
\end{remark}

\subsection{Assessing Fault Detection Capabilities}\label{ssec:fd-metrics}
A challenge in evaluating the Fault Detection capabilities of the proposed method lies in the fact that
we aim at \textit{anticipating} the fault events,
which however
start to become apparent, i.e., \textit{detectable}, at an unknown time preceding them.
If we do not take this fact into account, we
cannot distinguish if the algorithm
is anticipating a fault or yielding a false alarm.
To tackle this issue,
we consider a detection as successful if
a ``validity'' window of $w$ days prior to the fault
contains an alert.
On the remaining test days (i.e., outside of the validity window), we shall tally as many false alarms as days for which an alarm was raised.
Given a validity window $w$, we can thus define True Positives ($TP$), False Negatives ($FN$), False Positives ($FP$), and True Negatives ($TN$) as exposed in Table~\ref{tab:fd-glossary}.
\begin{table}[tb!]
    \begin{tabularx}{\textwidth}{l|X}
        \toprule

        Term & Definition\\
        \midrule
        Failure & A window of consecutive samples in which the machinery failed.\\ % our goal lies in raising an appropriate alert prior to its onset date \\
        Alert & A window of consecutive samples for which the detection algorithm raised a warning. \\
        Validity Window & All observations within a given number of days ($w$) prior to a fault; depending on the observed data, gaps may be present. \\
        True Positive ($TP$) & Number of failure instances for which at least an alert was active in the related validity window. \\
        False Negative ($FN$) & Number of failure instances for which no alert was raised in the related validity window\\
        False Positive ($FP$) & Number of recorded days outside validity windows, for which at least one alert was active. \\
        True Negative ($TN$) & Number of recorded days outside validity windows, for which no alert was active.\\
        \bottomrule
    \end{tabularx}
    \caption{Glossary of Definitions Relating to Failure Detection Metrics.}\label{tab:fd-glossary}
\end{table}
An illustrative example containing a tally of $TP$, $TN$, $FN$, and $FP$ is displayed on a fictitious sequence for $w=3$ in Figure~\ref{fig:fd-example}.
\begin{figure}[tb!]
    \includegraphics[width=\textwidth, trim={0 10cm 0 0}, clip]{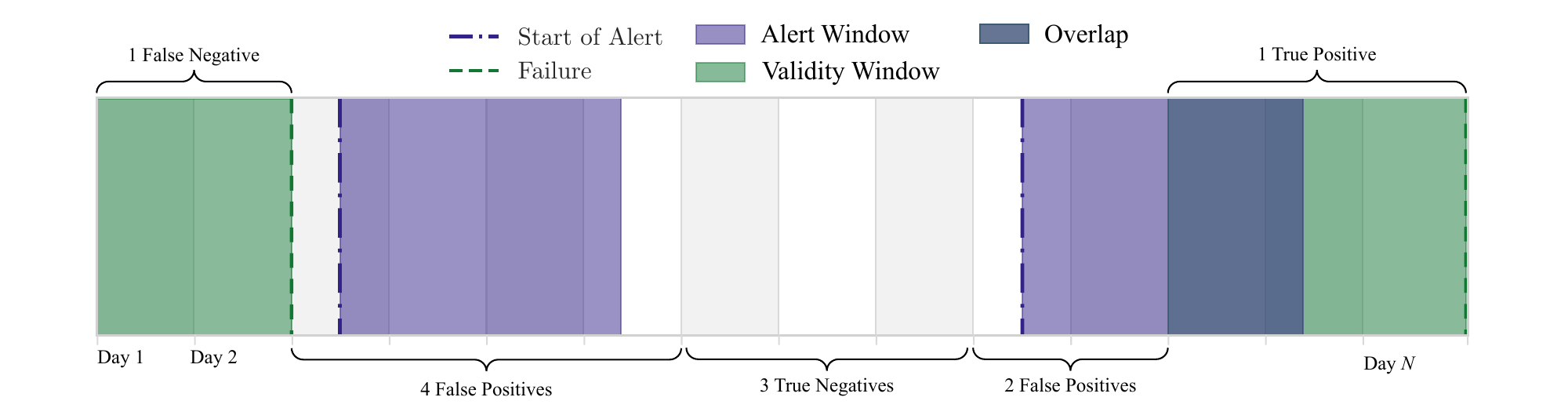}
    \caption{Illustrative example of Failure Detection-related concepts.}\label{fig:fd-example}
\end{figure}
We measure a model's performances via:
\begin{description}
    \item[Recall] The fraction of faults detected in the validity windows, $R = \frac{TP}{TP + FN}$;
    \item[Precision] The fraction of alerts that led to a detection, $P = \frac{TP}{TP + FP}$;
    \item[F1] The harmonic mean of precision and recall, i.e., $F = \frac{2RP}{R+P}$.
\end{description}
The ideal values for all metrics is 1, corresponding to perfect detection, whereas the worst value is 0 ($F = 0$ by convention if both $R$ and $P$ are zero).
As an example, the sequence in Figure~\ref{fig:fd-example} amounts to a Recall $R = 0.5$, a Precision $P = 0.1\bar{6}$, and an F1 of $F = 0.25$.
In our application, recall is more valuable than precision, as missed alarms are dangerous and costly.

%% file: sections/06_experiments.tex
\section{Experimental Validation on Machine Failure Data}\label{sec:experiments-azure}
In order to make our methodology replicable, we first apply it to a publicly available dataset.
The Azure Fault Detection Benchmark, first available as a part of Azure AI Notebooks for Predictive Maintenance, is an open dataset\footnote{\href{https://www.kaggle.com/datasets/arnabbiswas1/microsoft-azure-predictive-maintenance/data}{[Kaggle]~\texttt{microsoft-azure-predictive-maintenance}}} intended to support the development of predictive maintenance models by providing a structured collection of machine-related data.
The dataset also contains a comprehensive maintenance history, documented via timestamped instances of failures and repairs on specific machine components.
Telemetry data from sensors is recorded for a total of 100 machines, each identified by a unique machine identifier.
It consists primarily of sensor measurements that capture key performance parameters over time.
More specifically, the dataset includes time series collected at regular time intervals of four distinct signals: ``pressure'', ``rotate'', ``vibration'', and ``volt''.
The provided tabular files map telemetry readings to specific machines and timestamps, enabling temporal analysis of operational conditions.
We focus on the first machine (identified with label 1), and build the Train, Validation and Test sets as follows:
\begin{itemize}
	\item We extract all observations that are recorded no less than 5 days before and after each failure to create a fault-free Train and Validation Data;
	\item We subsample the fault-free data by taking 10\% of all observations (without altering their chronological order); this is due to the fact that the original data contains 876100 observations, which would make the computational complexity of performing full Bayesian inference practically infeasible;
	\item We select as Train set the first 200 samples of the resulting subsampled, fault-free data, and as Validation set the subsequent 100;
	\item Finally, we select all observations that were discarded in the first step which date after the last sample of the Validation set; we then keep all samples ranging from 5 days before to the date of each recorded failure.
\end{itemize}
While this procedure might seem cumbersome, it ensures that
we obtain Train and Validation sets that are as representative as possible as the fault- and wear-free condition of the machinery,
while the Test set contains a series of progressive onsets of the faults.
After data preprocessing (standard scaling of the indices), a distinct CoCoAFusE model is trained
on each of the aforementioned indices (``pressure'', ``rotate'', ``vibration'', and ``volt'') to describe the nominal conditional density of that HI with respect to the others.
As mentioned, our methodology does not preclude a multivariate response,
but for the sake of interpretability (see Section~\ref{sec:explainability})
we prefer to combine the outputs of each HI's analysis in a separate stage.
Figure~\ref{fig:azure-1-vibration} shows the Predictive Inference on the Train set for the Vibration Index, via:
\begin{description}
	\item[Top Subplot] The observations (black markers), the predictive density (shades of aqua superimposed by the mean as a line with dot markers on top and the 90\% intervals as simple lines of the same color);
	\item[Bottom Subplot] The gate probabilities for each Expert and the competitive Behaviour (hue shades corresponding to densities, superimposed by the mean as a line with dot markers on top and the 90\% intervals as plain lines, all color-coded to convey expert or behaviour pertinence).
\end{description}
\begin{figure}[tb!]
	\includegraphics[width=\textwidth]{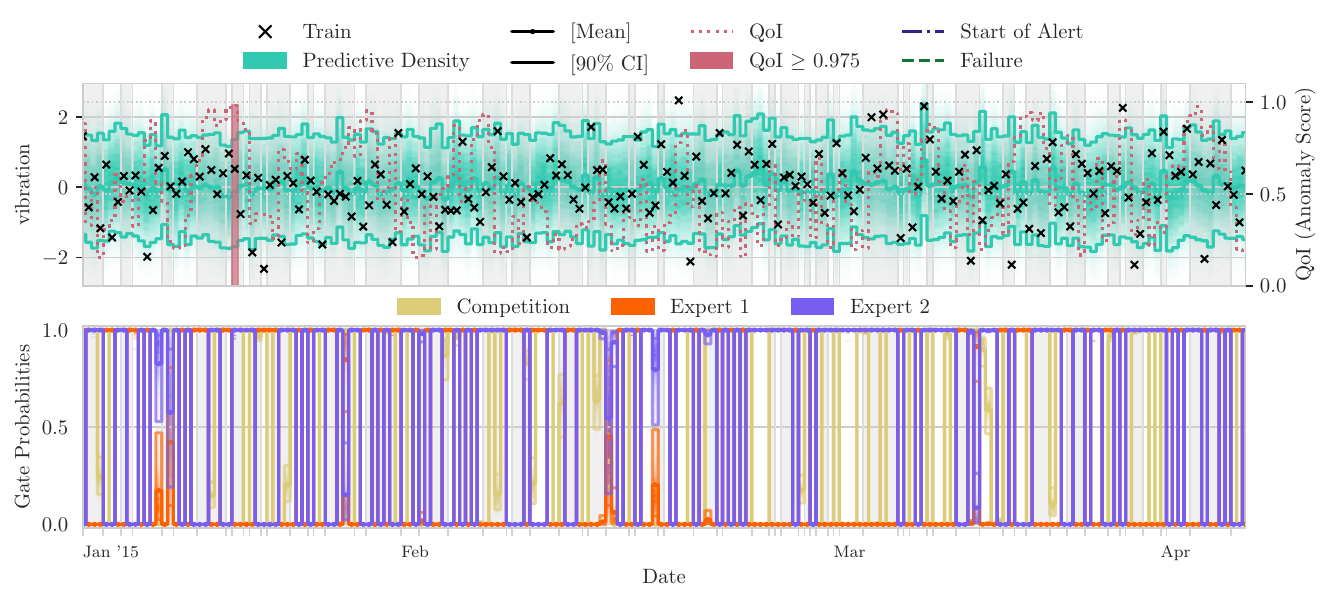}
	\caption{Predictive Inference for the Vibration Index of the Azure Train Dataset.}\label{fig:azure-1-vibration}
\end{figure}
We can observe how, under the assumed nominal conditions, the (normalized) Vibration samples appear to be concentrated around 0 with a consistent variance, encoded via CoCoAFusE by fusions of $M=2$ experts.

Each CoCoAFusE model can then be used to define an Anomaly Score as described in Section~\ref{sec:anomaly-scores} on windows of consecutive observations.
As mentioned, the $AS_{\mathcal{D}}$ measures how ``extreme'' a window of observations is, assuming that each response follows the conditional law postulated by the model.
By setting an
appropriate
threshold (we choose 97.5\%), it
is possible to
detect anomalies.
To further refine the detection logic and reduce false alarms, we implement a
filtering policy (``patience''), consisting of waiting
10 consecutive samples above threshold before raising an alarm,
which remains active while the threshold is exceeded.
We then evaluate the performance
in the Fault Detection task as described in Section~\ref{ssec:fd-metrics}.
The Anomaly Scores (dotted red lines, with blocked red overlay indicating threshold exceedance), Alarm onsets (blue dash-dotted vertical lines) for each Index, and the Recorded Failures (green dashed vertical lines, common to all Indices) can be seen for the Test set in Figure~\ref{fig:azure-1-test} (following the same convention as in Figure~\ref{fig:azure-1-vibration} concerning predictive inference), whereas a summary of the metrics on each Index are reported in Table~\ref{tab:azure-metrics}.%
\begin{figure}[t!]
	% trim={<left> <lower> <right> <upper>}
	\includegraphics[width=\textwidth, trim={-0.275cm 4.25cm 0 0}, clip]{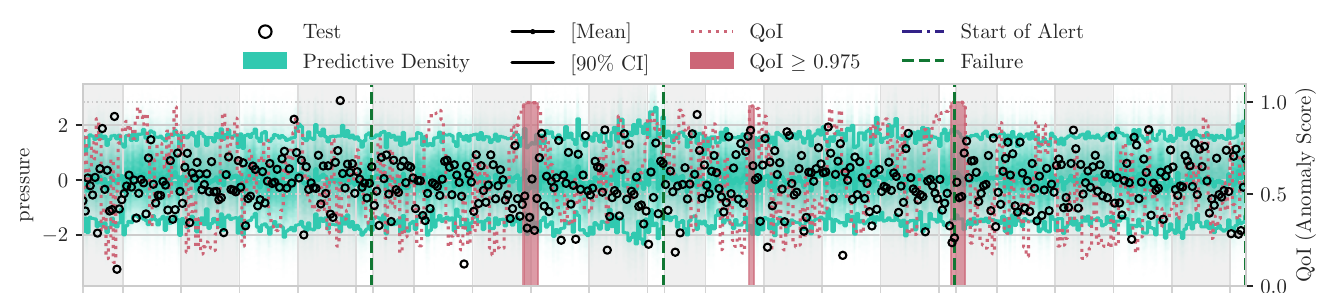}
	\includegraphics[width=\textwidth, trim={-0.275cm 4.25cm 0 1.3cm}, clip]{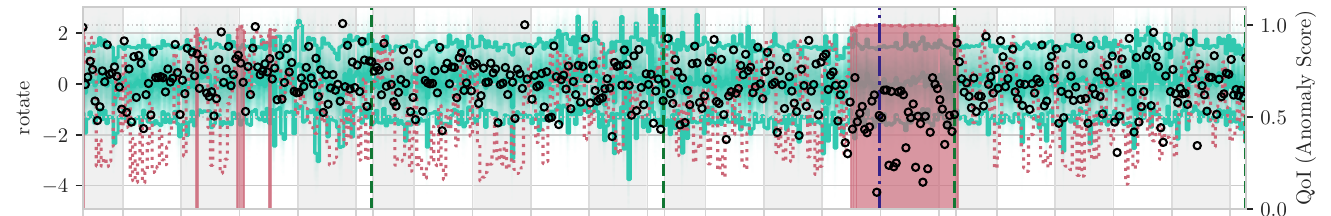}
	\includegraphics[width=\textwidth, trim={0 4.25cm 0 1.3cm}, clip]{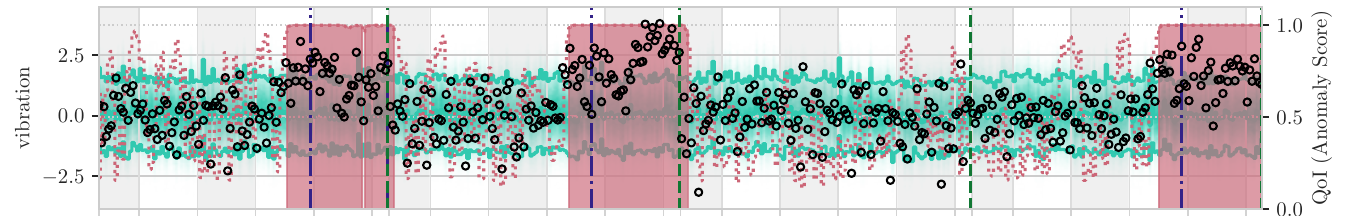}
	\includegraphics[width=\textwidth, trim={-0.275cm 3cm 0 1.3cm}, clip]{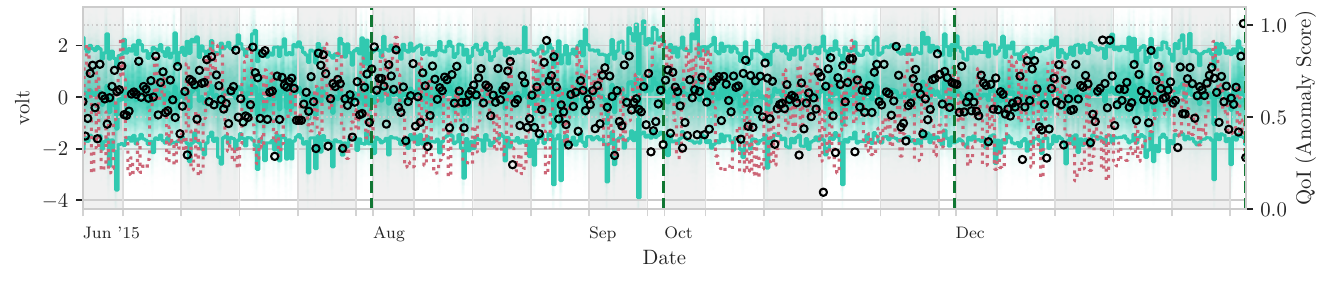}
	\caption{Predictive Inference Results for the Pressure (First Row), Rotate (Second Row), Vibration (Third Row), and Volt (Last Row) Indices of the Azure Test Dataset.}\label{fig:azure-1-test}
\end{figure}
\begin{table}[b!]
	\centering
	\input{tables/azure_1/merge/individual_indices.tex}
	\caption{Train and Test Metrics for each Index of the Azure Dataset.}\label{tab:azure-metrics}
\end{table}
Specifically, this Table reports the 95\% Credible Interval Coverage (CIC) on both the Train and Test sets, as well as the Precision, Recall and F1 of the Fault Detection on the Test Set. For the Fault Detection metrics, we consider the fault as detected if an alert was active in the day prior to it, whereas we consider an alarm as successful if it lasts at least until the day prior to a fault.
We observe that alerts occur more consistently in the very few days leading up to a fault, where the trends of the monitored indices significantly deviates from the expected behaviour (see Figure~\ref{fig:azure-1-test}).
The Rotate and Vibration indices, moreover, precisely complement each other in terms of the faults they are able to detect.
In practice, this could be the case for different types of fault, more easily detected via one Index compared to the others.

\paragraph{Pooled Detection}
We define a ``Pooled'' Anomaly Score as 1 if at least one Index exceeds the threshold, 0 otherwise.
Then, we apply the same detection logic as before, counting the consecutive samples for which the Pooled Anomaly Score evaluates to 1, with a patience of 10 samples prior to each Alert.
The Pooled Anomaly Score and new Alerts against the recorded faults in the Test set can be seen in Figure~\ref{fig:azure-1}, following the same convention as in the top subplot of Figure~\ref{fig:azure-1-vibration} (where, however, Test observations and Predictive Densities are omitted since multiple HIs are used).
\begin{figure}[tb!]
	\includegraphics[width=\textwidth, trim={0 3.25cm 0 0}, clip]{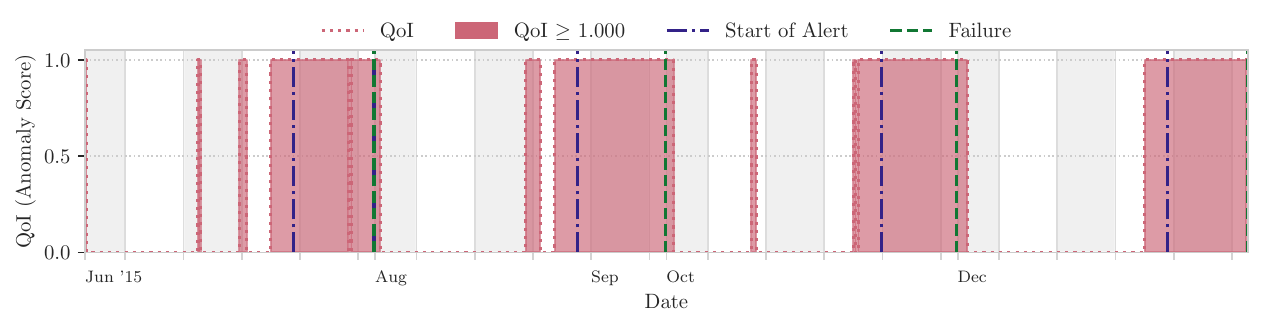}
	\caption{Anomaly Scores and Alerts on Pooled Indices of the Azure Dataset (Test Set).}\label{fig:azure-1}
\end{figure}
The consequent Fault Detection results, which are broken down according to progressive validity windows of length $w$ ranging from 1 (previous-day) to 4 days, are reported in Table~\ref{tab:azure-failure-detection}.
To enhance clarity, we group together consecutive days for which the metrics remain constant.
The number of samples that are specific to each validity window (or group thereof) are reported in the ``Samples in Range'' column.%
\begin{table}[b!]
	\centering
	\input{tables/azure_1/merge/pooled_indices.tex}
	\caption{Fault Detection Metrics for the Azure Test Dataset.}\label{tab:azure-failure-detection}
\end{table}

\section{Experimental Validation on Helicopter Transmission Data}\label{sec:experiments-helicopter}
We now move on to the actual Fault Detection applications of our methodology, analyzing datasets extracted from in-usage Helicopter machinery.
The datasets in this work, which have been previously analyzed in~\citet{2022leoniNewComprehensiveMonitoring}, were collected over the span of three years (March 2015 - December 2018) on two helicopters of the same family.
The onboard instrumentation, for each helicopter, records accelerometric signals from 23 accelerometers (of which 18 are monoaxial, 2 biaxial, and 3 triaxial) monitoring a total of 88 working components.
The raw data are collected at 40kHz, from which the onboard mission computers extract relevant Health Indices over the span of several minutes of recording.
Each accelerometer can sense the vibration of one or more components, and each accelerometer-component pairing (called ``acquisition identifier'') is associated with a minimum of 1 up to a maximum of 18 distinct types of HIs, relating to the type of fault condition that is being monitored.
Alongside the previous, the other variables in Table~\ref{tab:dataset-variables} are recorded with each instance.
\begin{table}[tb!]
	\centering
	\begin{tabularx}{\textwidth}{X|l}
		\toprule
		Variable & Details\\\midrule
        $[$ Alphanumeric Identifier $]$& Health Index, e.g., ``ENHKUR'', ``S1R''\\
		Engine 1 Torque & \\
		Engine 2 Torque & \\
		MGB 1 Oil Temperature & MGB: Main Gearbox Bearing \\
		MGB 2 Oil Temperature & \\
		Main Rotor Torque & \\
		Tail Rotor Torque & \\
		NR & NR: Rotor Speed\\
		IAS & IAS: Measured airspeed  \\
		Pitch & \\
		Roll & \\
		MGB Oil Pressure & \\
		MGB Oil Temperature& \\
		IGB Oil Temperature& IGB: Intermediate Gearbox Bearing\\
		TGB Oil Temperature& TGB: Rear Gearbox Bearing\\
		\bottomrule
	\end{tabularx}
	\caption{List of flight parameters in the datasets}\label{tab:dataset-variables}
\end{table}
The result is a large number of distinct HIs, which are univocally individuated by the accelerometer, target component, and Index type.
Each HI pertains to the time history of a specific acquisition identifier analyzed in one of the following ``acquisition modes''.
\begin{description}
	\item[Time Average:] time domain analyses to highlight improper meshing, distributed teeth surface degradation, fatigue cracks in gears and shafts.
	\item[Envelope:] localized pits, spalls, debris over inner/outer race energy components, rolling elements, and cage components.
	\item[Time History:] residual analyses to point out localized or distributed pits, spalls, or cracks over gears and bearings.
	\item[Auto Spectrum:] cepstrum analysis to identify localized or distributed pits, spalls, or cracks over outer gears and bearings.
\end{description}
As discussed in Section~\ref{sec:problem-statement}, we work on a per-index (hence also per-acquisition identifier) basis in order to obtain a conditional density on flight parameters.
Since we aim at characterizing the nominal behaviour of the indices, we will select an initial portion of faultless data (as provided by domain experts) to learn the conditional density.
We then use the remaining data to evaluate the Fault Detection algorithm based on the concepts from Section~\ref{sec:anomaly-scores}.

\subsection{First Helicopter: Swashplate Damage}\label{ssec:hc_3-26}
From domain knowledge, we know that the First Helicopter underwent a Swashplate fault discovered on January 29, 2017.
This fault has been studied previously in~\citet{2022leoniNewComprehensiveMonitoring}, whereby the detection is triggered by anomalous
patterns in the vibration time series related to a specific accelerometer-component pairing (labeled as acquisition identifier 26).
In an attempt to validate our method against an established baseline, which is a rarity in this context, we perform our analyses on the same HIs for said acquisition identifier (26) as in~\citet{2022leoniNewComprehensiveMonitoring}.
We build our training and validation sets as the pool of observations until October 31, 2016, i.e., roughly three months before the recorded fault.
Coherently, we use all observations from November 1, 2016, until January 29, 2017, as test dataset for the Fault Detection algorithm, considering as faulty all observations from the 7 days prior to the fault discovery.
Like in the previous case, a distinct CoCoAFusE model is trained on each HI related to accelerometer-component pairing 26 to describe the nominal conditional density of that HI with respect to other HIs, as well as additional flight parameters that were recorded (see Table~\ref{tab:dataset-variables}).
Figure~\ref{fig:HC_3-26-enhkur} depicts the Predictive Inference on one of the HIs (ENHKUR) on the different folds of the dataset, following the same conventions as in Figure~\ref{fig:azure-1-vibration}.
\begin{figure}[p!]
	\includegraphics[width=\textwidth]{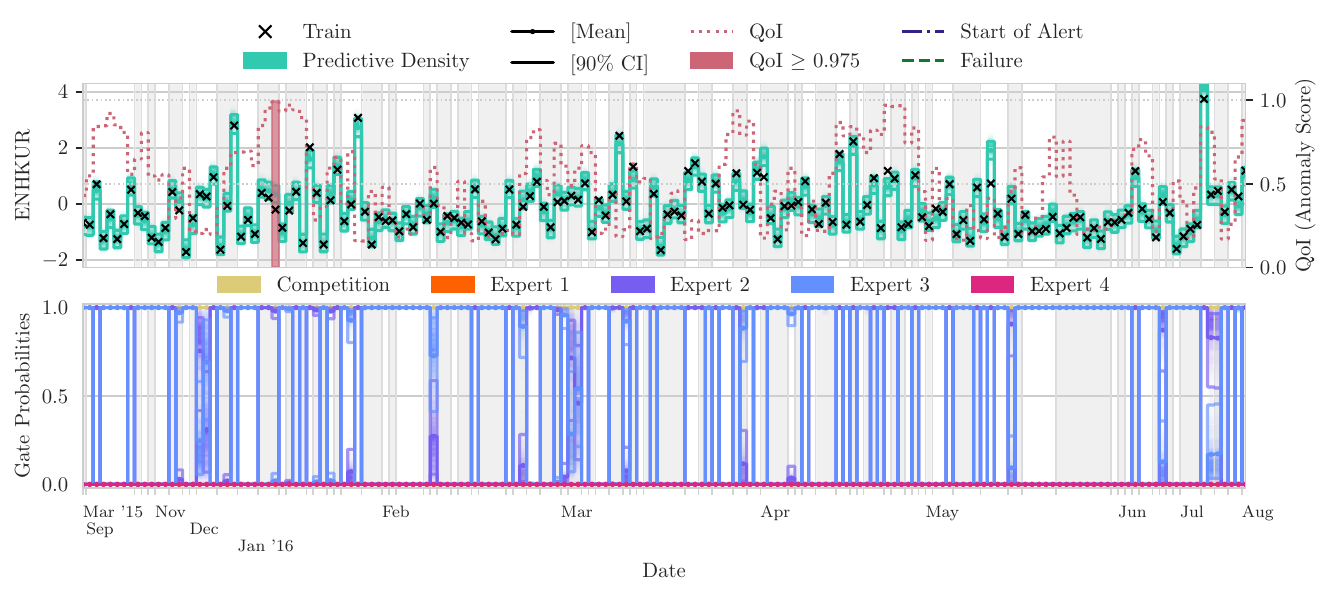}
	\includegraphics[width=\textwidth]{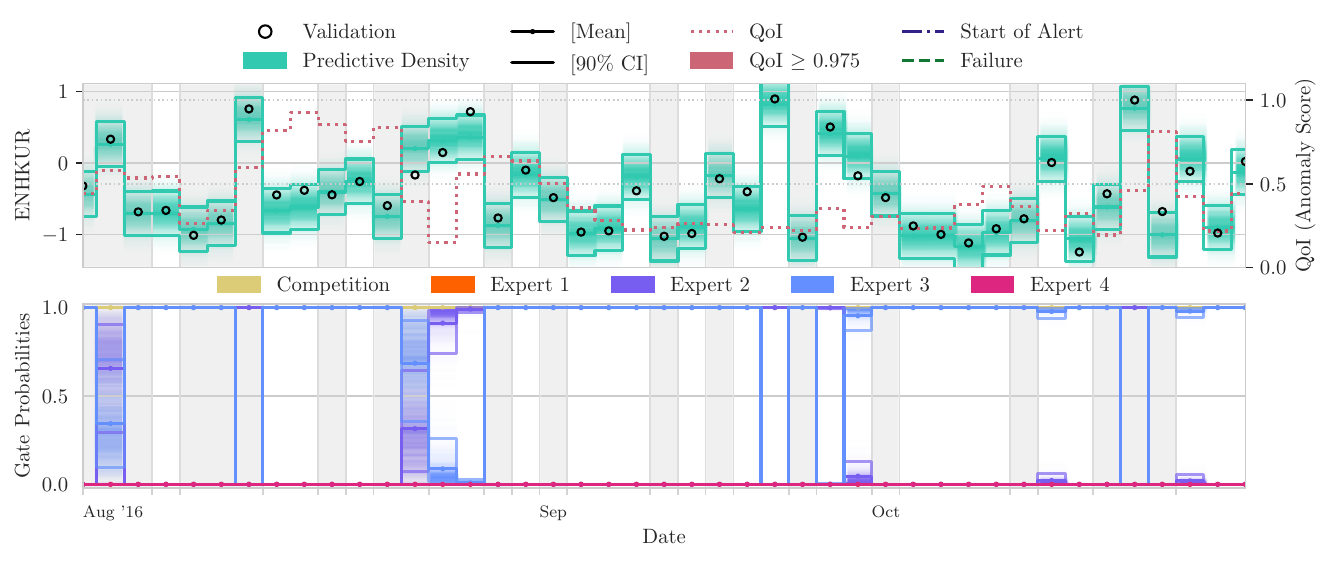}
	\includegraphics[width=\textwidth]{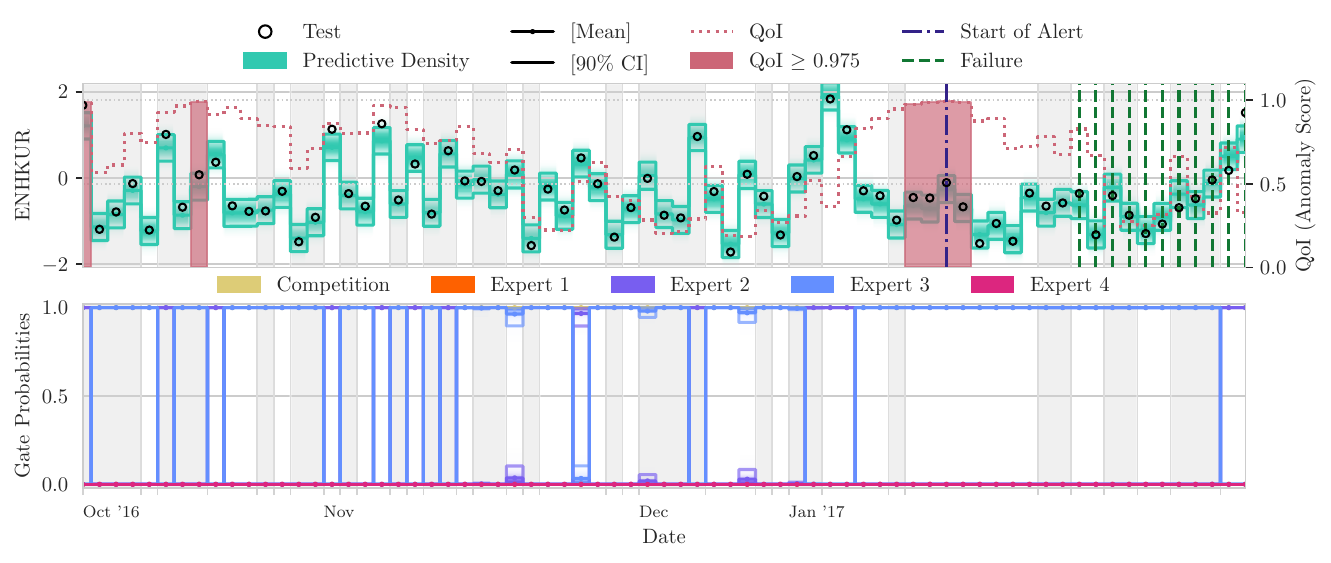}
	\caption{Predictive Inference for the ENHKUR Health Index of First Helicopter on the Train (top), Validation (middle), and Test (bottom) Sets.}\label{fig:HC_3-26-enhkur}
\end{figure}
Following the same procedure as for the Azure dataset, we define the Anomaly Scores from Section~\ref{sec:anomaly-scores} and set a threshold of 97.5\% to characterize anomalies.
Compared to the Azure case, we implement a less conservative 
patience filtering policy of 3 samples before raising alarms,
in light of the lower temporal resolution.

The comprehensive results for each HI are detailed in Table~\ref{tab:hc_3-26-metrics}, where we report the 95\% Credible Interval Coverage (CIC) on both the Train and Test sets, as well as the Precision, Recall and F1 of the Fault Detection on the Test Set (at a fixed validity window of $w = 60$ days prior to failure).
We can observe how two HIs allow to raise a successful alarm prior to the failure week, and how this is associated with a lower CIC on the test set, as expected.
\begin{table}[t!]
	\centering
	\input{tables/HC_3_26/merge/individual_indices.tex}
	\caption{Train and Test Metrics for each Health Index of the First Helicopter.}\label{tab:hc_3-26-metrics}
\end{table}%

\paragraph{Pooled Detection}
Also in this case, we combine the Anomaly Scores of different HIs, since it is not possible to dismiss underperforming HIs (see Table~\ref{tab:hc_3-26-metrics}) for unrecorded types of faults.
We set the Pooled Anomaly Score to 1 if at least two HIs exceed the respective threshold, and raise an Alarm after 3 consecutive samples with a Pooled score equal to 1.
The logic behind this choice is to gather a ``consensus'' from the HIs on the observations that might be anomalous, reducing false positives.
For the sake of presentation, we also set the Pooled Score to 0.5 if a single HI is above threshold, as shown in Figure~\ref{fig:HC_3-26}.
\begin{figure}[tb!]
    \includegraphics[width=\textwidth, trim={0 3cm 0 0}, clip]{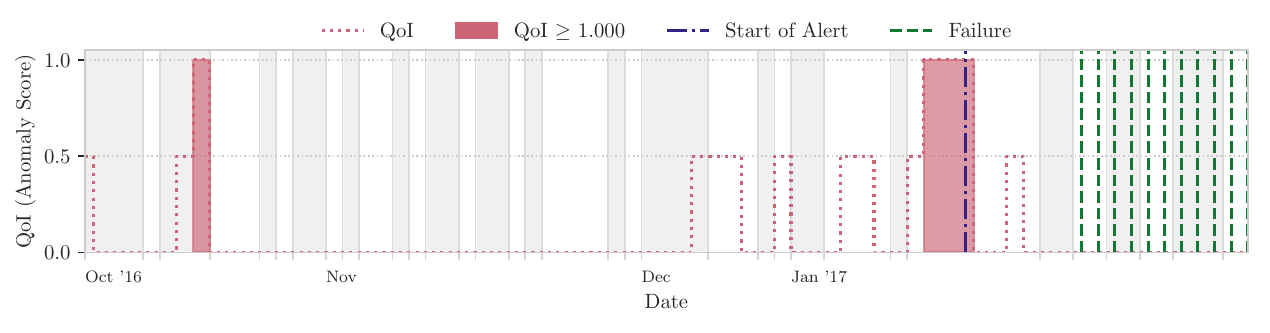}
    \caption{CoCoAFusE Anomaly Scores and Alerts on Pooled Health Indices of the First Helicopter (Test Set).}\label{fig:HC_3-26}
\end{figure}
The consequent Fault Detection results, now with validity windows ranging from 1 (previous-day) to 89, are reported in Table~\ref{tab:hc_3-26-failure-detection}.
The same metrics, computed on the algorithm presented in~\citet{2022leoniNewComprehensiveMonitoring}, are instead in table Table~\ref{tab:hc_3-26-failure-detection-jessica}.
Notice that some gaps in the data are present (e.g., 2 to 16 days prior to the fault), for which the metrics are not reported.
\begin{table}[tb!]
	\centering
	\input{tables/HC_3_26/merge/pooled_indices.tex}
	\caption{Fault Detection Metrics for the First Helicopter (CoCoAFusE).}\label{tab:hc_3-26-failure-detection}
\end{table}
\begin{table}[tb!]
	\centering
	\input{tables/HC_3_26/Jessica/pooled_indices.tex}
	\caption{Fault Detection Metrics for the First Helicopter (from~\citet{2022leoniNewComprehensiveMonitoring}).}\label{tab:hc_3-26-failure-detection-jessica}
\end{table}%
Also, the detection algorithm fails to raise an alert from the very day prior to the fault week, which explains the first row of Table~\ref{tab:hc_3-26-failure-detection}.

\subsection{Second Helicopter: Gear Bearing Fault}\label{ssec:69019-126}
The Second Helicopter in this study underwent a Gear Bearing fault discovered on May 25, 2018.
Again,
we analyze the
dataset
and correctly
identify the
reported
fault. In~\citet{2022leoniNewComprehensiveMonitoring},
the
acquisition
identifier
associated to the
detection
was 126.
Thus, we
focus on this
same
identifier.
We coherently
build the training and validation sets as the observations until December 31, 2017 (almost five months before the recorded fault).
We use all observations from January 1, 2018, until May 25, 2018, as test dataset for the Fault Detection algorithm, again considering as faulty all observations from the 7 days prior to the fault discovery.
We repeat verbatim the analysis as per the First Helicopter.
Predictive Inference on the OREB1A HI is portrayed in Figure~\ref{fig:69019-126-oreb1a}.
\begin{figure}[p!]
	\includegraphics[width=\textwidth]{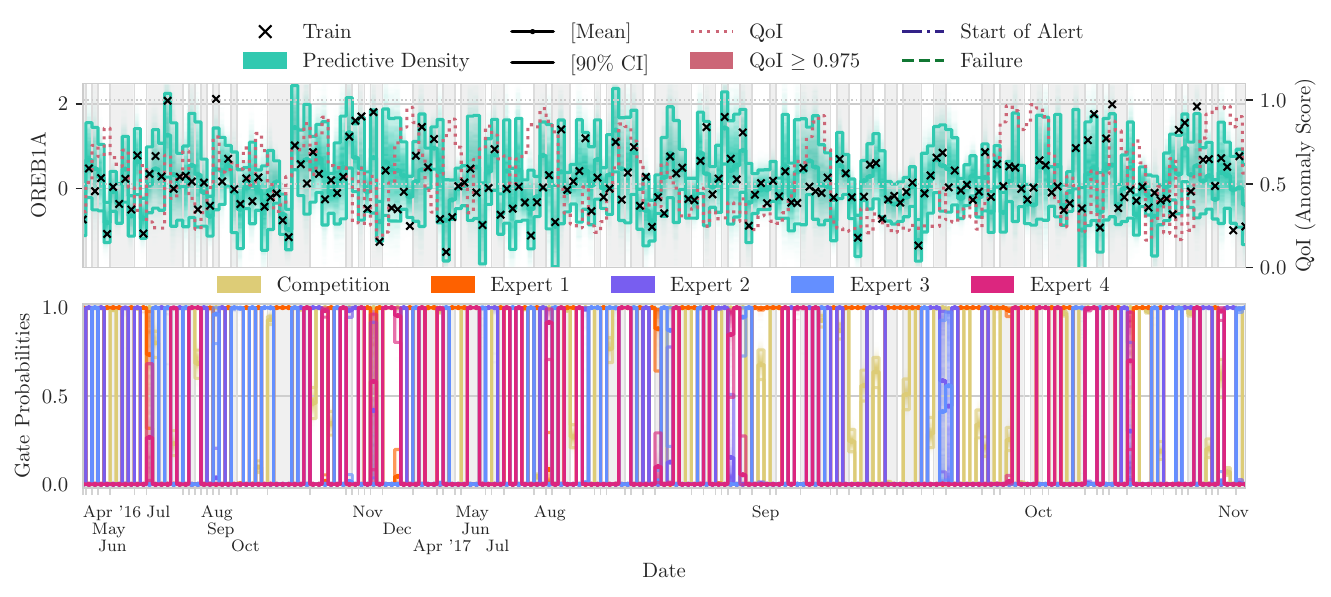}
	\includegraphics[width=\textwidth]{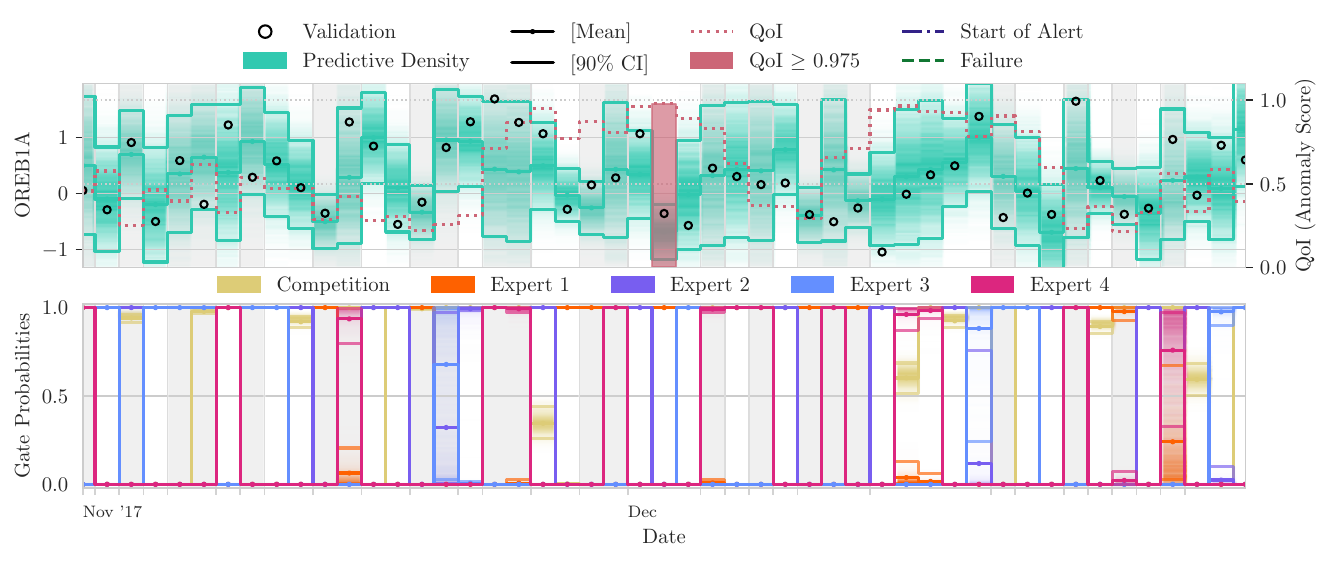}
	\includegraphics[width=\textwidth]{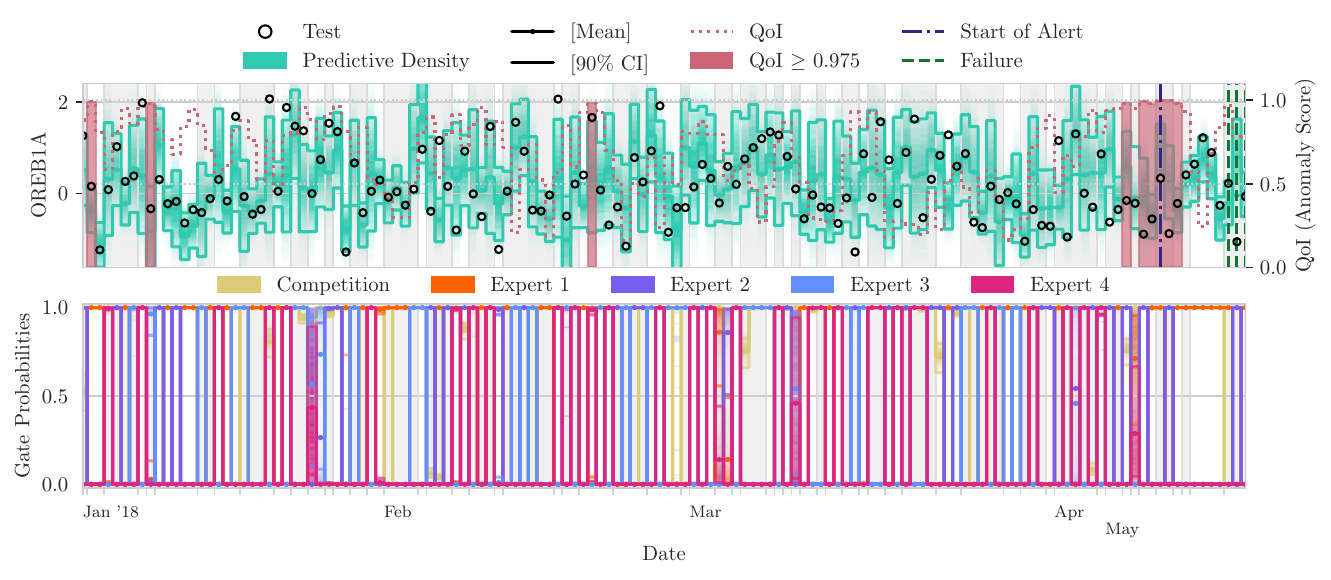}
	\caption{Predictive Inference for the OREB1A Health Index of the Second Helicopter on the Train (top), Validation (middle), and Test (bottom) Sets.}\label{fig:69019-126-oreb1a}
\end{figure}
The results for individual HIs
are reported in Table~\ref{tab:69019-126-metrics}.
The Fault Detection metrics on Pooled HIs for different validity windows on the Test set are reported in Table~\ref{tab:69019-126-failure-detection}, whereas a visualization of the Pooled Anomaly Score, alerts and faults occurrences can be found in Figure~\ref{fig:69019-126}.
For comparison, see the same metrics, computed on the algorithm presented in~\citet{2022leoniNewComprehensiveMonitoring}, in Table~\ref{tab:69019-126-failure-detection-jessica}.
\begin{table}[t!]
    \centering
    \input{tables/69019_126/merge/individual_indices.tex}
    \caption{Train and Test Metrics for each Health Index of the Second Helicopter.}\label{tab:69019-126-metrics}
\end{table}
In this case, we can observe a higher precision than in the methodology described in~\citet{2022leoniNewComprehensiveMonitoring}, due to our filtering policy.%
\begin{table}[tb!]
	\centering
	\input{tables/69019_126/merge/pooled_indices.tex}
	\caption{Fault Detection Metrics for the Second Helicopter (CoCoAFusE).}\label{tab:69019-126-failure-detection}
\end{table}%
\begin{table}[tb!]
	\centering
	\input{tables/69019_126/Jessica/pooled_indices.tex}
	\caption{Fault Detection Metrics for the Second Helicopter (from~\citet{2022leoniNewComprehensiveMonitoring}).}\label{tab:69019-126-failure-detection-jessica}
\end{table}%
\begin{figure}[tb!]
	\includegraphics[width=\textwidth, trim={0 3cm 0 0}, clip]{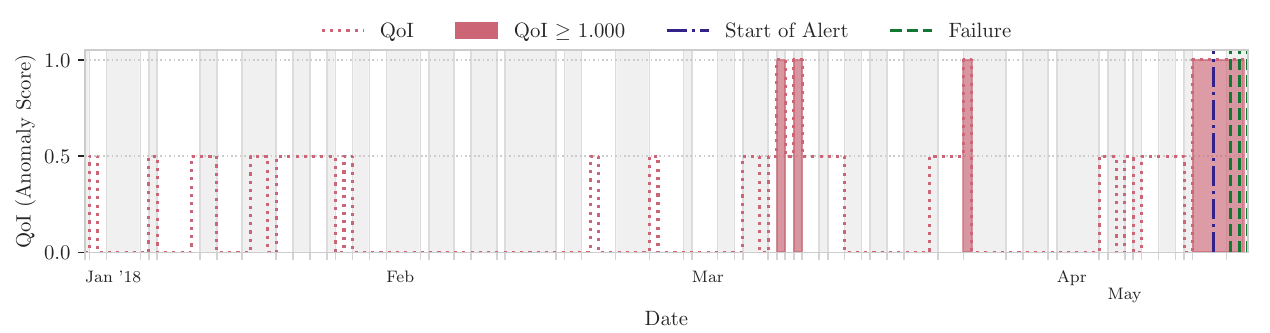}
	\caption{Test Anomaly Scores and Alerts on Pooled Indices of the Second Helicopter.}\label{fig:69019-126}
\end{figure}%
\subsection{Explainability Analysis}\label{sec:explainability-analysis}
In Section~\ref{sec:explainability}, we have discussed some approaches to inspect the models fitted via CoCoAFusE.
We deem that
this
plays a crucial
role in
algorithm
selection,
particularly in
the considered
Fault Detection
context.
This is because being able to effectively raise alarms on a component or machinery does not, by itself, allow users to gain trust in the model.
Also,
in case of
misdetection,
interpretability
becomes key
to improve
the model's
performance
and understand
its behavior.
Because of this, we believe that CoCoAFusE has a competitive edge compared to many standard Machine Learning models in light of its interpretability-oriented design.
In the follwing, we present Explainability Analyses on the ENHKUR model for the First Helicopter (see Figure~\ref{fig:HC_3-26-enhkur}) and the OREB1A model for the Second (see Figure~\ref{fig:69019-126-oreb1a}).
\paragraph{ENHKUR Model (First Helicopter)}
We use the procedure described in Section~\ref{sec:explainability} to obtain a low-dimensional set of coordinates upon which to define a grid of points for inference.
Since, of the ENHKUR Model's input features, only an affine transformation of the ENHM6 HI enters the gating, the resulting set of coordinates is one-dimensional and affine in ENHM6, and the grid consists of equispaced scalars.
We can use such newly-defined quantity (affine in ENHM6), which we name ``Component 1'', to create a fictitious dataset for inference which lays on a straight line within the 22-dimensional space of features
(all of the fictitious variables except for ENHM6 are constant and set to their mean value over the Train set by our upscaling procedure).
The inference on such data against the actual Train observations is displayed in Figure~\ref{fig:HC_3_26-enhkur-explanations}.
\begin{figure}[b!]
	\includegraphics[width=\textwidth]{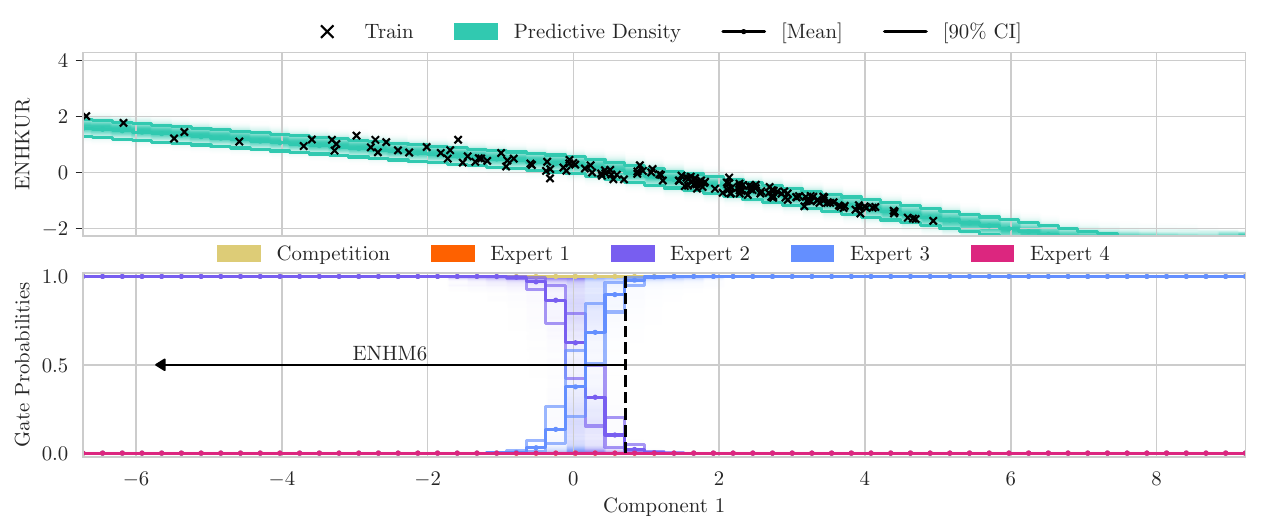}
	\caption{Explanatory Analysis on the ENHKUR Model for the First Helicopter.}\label{fig:HC_3_26-enhkur-explanations}
\end{figure}
We can
derive two conclusions:
\begin{itemize}
	\item All but two of the experts (2 and 3) are never used, since the rank of the Expert Gate's coefficient matrix is 2, which means that we could in principle have trained the very same model by imposing $M = 2$;
	\item One of the active experts (Expert 2) handles data with high ENHM6 (see the black arrow in Figure\ref{fig:HC_3_26-enhkur-explanations}, representing the direction of increase), whereas the other (Expert 3) is primarily engaged for low ENHM6.
\end{itemize}
Thorugh plots such as Figure~\ref{fig:HC_3_26-enhkur-explanations}, we can visually dissect the ENHKUR model into its linear submodels and thus better inspect its predictions,
evaluate counterfactuals,
or refine data collection during an experimental campaign.
\paragraph{OREB1A Model (Second Helicopter)}
In this second case, the procedure in Section~\ref{sec:explainability} can only be used after manipulating the Gates' coefficients matrix $\bm{A}_\mathcal{G}$.
In particular, we use SVD to first transform it into a $2\times (n + 1)$ matrix, where $n = 22$ is the number of features, and then define the 2D subspace in which we can define a regular grid of point for visualization.
Within this subspace, whose coordinates we label ``Coordinate 1'' and ``Coordinate 2'', corresponding to the 2 largest singular values of $\bm{A}_0$,
we can embed a regular grid of points to perform inference.
Figure~\ref{fig:69019_126-oreb1a-explanations} displays the results:
\begin{figure}[p!]
	\subfigure[Average Expert Activation]{\includegraphics[width=0.5\textwidth]{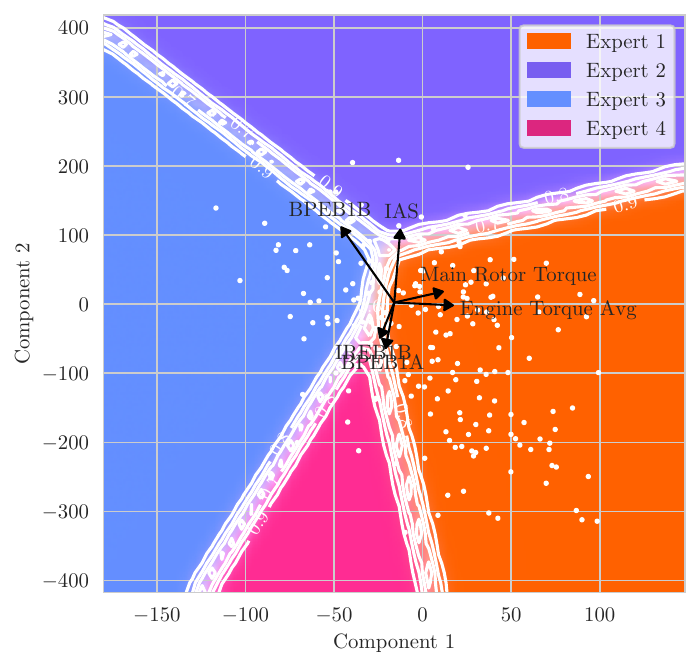}\label{fig:69019_126-oreb1a-explanation-expert-activation}}
	\subfigure[Predictive Mean]{\includegraphics[width=0.5\textwidth]{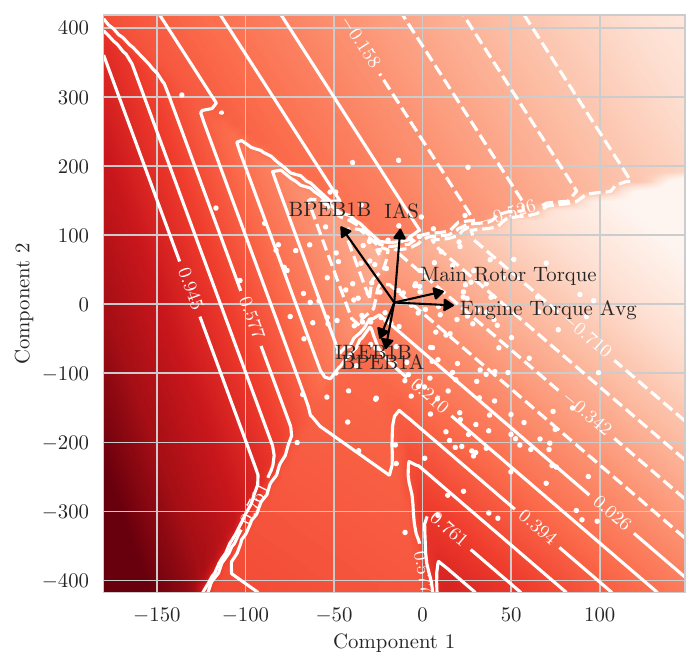}\label{fig:69019_126-oreb1a-explanation-pred-mean}}
	\subfigure[Predictive Standard Deviation]{\includegraphics[width=0.5\textwidth]{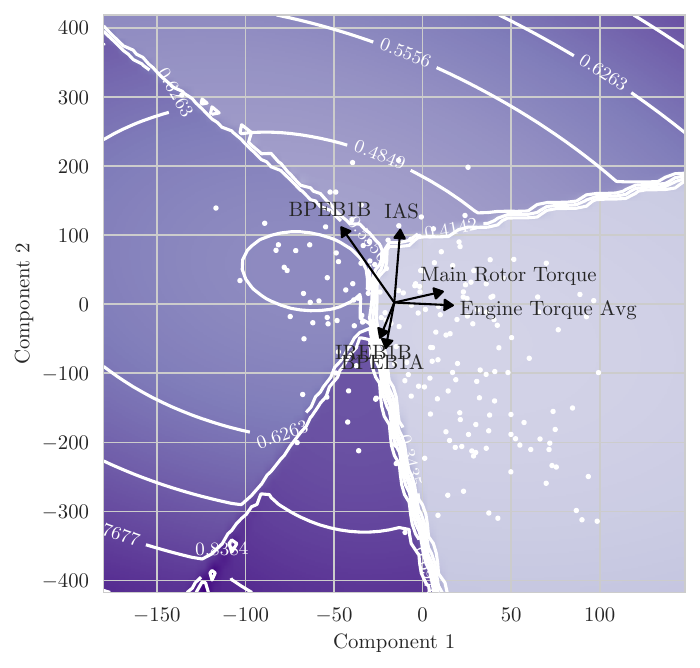}\label{fig:69019_126-oreb1a-explanation-pred-std}}
	\subfigure[Pred. Mean vs. Avg. Expert Activation]{\includegraphics[width=0.5\textwidth]{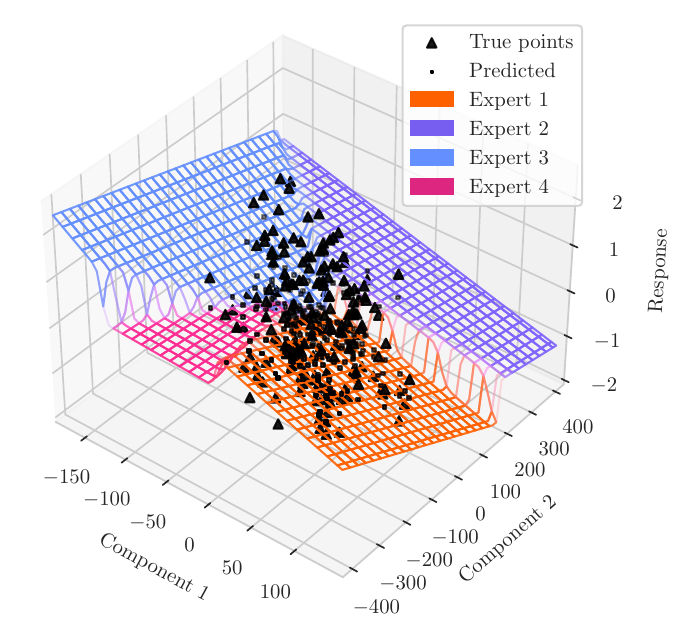}\label{fig:69019_126-oreb1a-explanation-surf}}
	\caption{Explanatory Analysis on the OREB1A Model for the Second Helicopter.}\label{fig:69019_126-oreb1a-explanations}
\end{figure}
\begin{enumerate}
	\item Figure~\ref{fig:69019_126-oreb1a-explanation-expert-activation} portrays the average activation of the engaged Expert in the 2D subspace (color-coded, with countour lines representing the value of the average activation); superimposed are black arrows representing the directions of maximal change of each feature (i.e., the projection of the original data coordinates) on the plane;
	\item Figure~\ref{fig:69019_126-oreb1a-explanation-pred-mean} portrays the predictive mean in the 2D subspace (shades of red, with contour lines representing the value);
	\item Figure~\ref{fig:69019_126-oreb1a-explanation-pred-std} portrays the predictive standard error in the 2D subspace (shades of purple, with contour lines representing the value);
	\item Figure~\ref{fig:69019_126-oreb1a-explanation-surf} contains a 3D plot of the response against the two Components, with the triangle-shaped black markers representing actual observations from the Train set, the small black squares designating the corresponding Posterior Mean predictions, the wireframe indicating the Predictive Mean on each location, color-coded according to the Expert.
\end{enumerate}
Figure~\ref{fig:69019_126-oreb1a-explanations} helps interpreting the CoCoAFusE model as a gated superimposition of the linear Experts.
Indeed, the arrows in Figures~\ref{fig:69019_126-oreb1a-explanation-expert-activation}-\ref{fig:69019_126-oreb1a-explanation-pred-std} account for effects of individual input features on Expert allocations (see Figure~\ref{fig:69019_126-oreb1a-explanation-expert-activation}), predictive mean (Figure~\ref{fig:69019_126-oreb1a-explanation-pred-mean}), and uncertainty (Figure~\ref{fig:69019_126-oreb1a-explanation-pred-std}).

%% file: tables/azure_1/merge/individual_indices.tex
\begin{tabularx}{\textwidth}{Xccccc}
	\toprule

	& Train & \multicolumn{4}{c}{Test} \\
	\cmidrule(lr){2-2}\cmidrule(lr){3-6}

	& & & \multicolumn{3}{c}{1 Day to Event} \\
	\cmidrule(lr){4-6}

	& $_\uparrow^\downarrow$ 95\% CIC
	& $_\uparrow^\downarrow$ 95\% CIC
	& $\uparrow$ Prec. [\%]
	& $\uparrow$ Rec. [\%]
	& $\uparrow$ F1 [\%]\\

	\midrule

	pressure & 
		0.96\textpm0.01 &
		0.95\textpm0.01 &
		0.00 &
		0.00 &
		0.00\\

	rotate & 
		0.96\textpm0.01 &
		0.93\textpm0.01 &
		100.00 &
		25.00 &
		40.00\\

	vibration & 
		0.94\textpm0.02 &
		0.87\textpm0.01 &
		100.00 &
		75.00 &
		85.71\\

	volt & 
		0.95\textpm0.02 &
		0.96\textpm0.01 &
		0.00 &
		0.00 &
		0.00\\

	\bottomrule
\end{tabularx}

%% file: tables/azure_1/merge/pooled_indices.tex
\begin{tabularx}{\textwidth}{Xcccc}
	\toprule

	Days to Event
	& Samples in Range
	& $\uparrow$ Prec. [\%]
	& $\uparrow$ Rec. [\%]
	& $\uparrow$ F1 [\%]\\

	\midrule

	4-1 &
	384 &
	100.00 &
	100.00 &
	100.00\\

	\bottomrule
\end{tabularx}

%% file: tables/HC_3_26/merge/individual_indices.tex
\begin{tabularx}{\textwidth}{Xccccc}
	\toprule

	& Train & \multicolumn{4}{c}{Test} \\
	\cmidrule(lr){2-2}\cmidrule(lr){3-6}

	& & & \multicolumn{3}{c}{60 Days to Event} \\
	\cmidrule(lr){4-6}

	& $_\uparrow^\downarrow$ 95\% CIC
	& $_\uparrow^\downarrow$ 95\% CIC
	& $\uparrow$ Prec. [\%]
	& $\uparrow$ Rec. [\%]
	& $\uparrow$ F1 [\%]\\

	\midrule

	ENHKUR & 
		0.96\textpm0.01 &
		0.94\textpm0.03 &
		100.00 &
		100.00 &
		100.00\\

	ENHM6 & 
		0.96\textpm0.02 &
		0.94\textpm0.03 &
		0.00 &
		0.00 &
		0.00\\

	ENHP2P & 
		0.95\textpm0.02 &
		1.00\textpm0.00 &
		0.00 &
		0.00 &
		0.00\\

	ENHSTD & 
		0.98\textpm0.01 &
		0.99\textpm0.01 &
		0.00 &
		0.00 &
		0.00\\

	S1R & 
		0.96\textpm0.02 &
		0.90\textpm0.04 &
		0.00 &
		0.00 &
		0.00\\

	S2R & 
		0.98\textpm0.01 &
		0.90\textpm0.04 &
		100.00 &
		100.00 &
		100.00\\

	SnRA & 
		0.94\textpm0.02 &
		0.94\textpm0.03 &
		0.00 &
		0.00 &
		0.00\\

	\bottomrule
\end{tabularx}

%% file: tables/HC_3_26/merge/pooled_indices.tex
\begin{tabularx}{\textwidth}{lXccc}
	\toprule

	Days to Event
	& Samples in Range
	& $\uparrow$ Prec. [\%]
	& $\uparrow$ Rec. [\%]
	& $\uparrow$ F1 [\%]\\

	\midrule

	1 &
	2 &
	0.00 &
	0.00 &
	0.00\\
	
	89-17 &
	46 &
	100.00 &
	100.00 &
	100.00\\

	\bottomrule
\end{tabularx}

%% file: tables/HC_3_26/Jessica/pooled_indices.tex
\begin{tabularx}{\textwidth}{lXccc}
	\toprule

	Days to Event
	& Samples in Range
	& $\uparrow$ Prec. [\%]
	& $\uparrow$ Rec. [\%]
	& $\uparrow$ F1 [\%]\\

	\midrule

	89-1 &
	48 &
	100.00 &
	100.00 &
	100.00\\

	\bottomrule
\end{tabularx}

%% file: tables/69019_126/merge/individual_indices.tex
\begin{tabularx}{\textwidth}{Xccccc}
	\toprule

	& Train & \multicolumn{4}{c}{Test} \\
	\cmidrule(lr){2-2}\cmidrule(lr){3-6}

	& & & \multicolumn{3}{c}{60 Days to Event} \\
	\cmidrule(lr){4-6}

	& $_\uparrow^\downarrow$ 95\% CIC
	& $_\uparrow^\downarrow$ 95\% CIC
	& $\uparrow$ Prec. [\%]
	& $\uparrow$ Rec. [\%]
	& $\uparrow$ F1 [\%]\\

	\midrule

	BPEB1A & 
		0.96\textpm0.01 &
		0.87\textpm0.03 &
		50.00 &
		100.00 &
		66.67\\

	BPEB1B & 
		0.98\textpm0.01 &
		0.86\textpm0.03 &
		0.00 &
		0.00 &
		0.00\\

	FTEB1A & 
		1.00\textpm0.00 &
		0.99\textpm0.01 &
		0.00 &
		0.00 &
		0.00\\

	FTEB1B & 
		0.98\textpm0.01 &
		0.97\textpm0.01 &
		0.00 &
		0.00 &
		0.00\\

	IREB1A & 
		0.99\textpm0.01 &
		0.93\textpm0.02 &
		100.00 &
		100.00 &
		100.00\\

	IREB1B & 
		0.94\textpm0.02 &
		0.91\textpm0.02 &
		0.00 &
		0.00 &
		0.00\\

    OREB1A & 
		0.96\textpm0.01 &
		0.86\textpm0.03 &
		100.00 &
		100.00 &
		100.00\\

	OREB1B & 
		0.92\textpm0.02 &
		0.93\textpm0.02 &
		0.00 &
		0.00 &
		0.00\\

	SHEB1A & 
		1.00\textpm0.00 &
		0.98\textpm0.01 &
		0.00 &
		0.00 &
		0.00\\

	SHEB1B & 
		1.00\textpm0.00 &
		0.98\textpm0.01 &
		0.00 &
		0.00 &
		0.00\\

	\bottomrule
\end{tabularx}

%% file: tables/69019_126/merge/pooled_indices.tex
\begin{tabularx}{\textwidth}{lXccc}
	\toprule

	Days to Event
	& Samples in Range
	& $\uparrow$ Prec. [\%]
	& $\uparrow$ Rec. [\%]
	& $\uparrow$ F1 [\%]\\

	\midrule

	89-4 &
	68 &
	100.00 &
	100.00 &
	100.00\\

	\bottomrule
\end{tabularx}

%% file: tables/69019_126/Jessica/pooled_indices.tex
\begin{tabularx}{\textwidth}{lXccc}
	\toprule

	Days to Event
	& Samples in Range
	& $\uparrow$ Prec. [\%]
	& $\uparrow$ Rec. [\%]
	& $\uparrow$ F1 [\%]\\

	\midrule

	70-4 &
	49 &
	50.00 &
	100.00 &
	66.67\\
	
	89-71 &
	19 &
	100.00 &
	100.00 &
	100.00\\

	\bottomrule
\end{tabularx}

%% file: sections/99_appendices.tex
\section{Proofs}\label{app:proofs}
\begin{proof}[Proof (of Proposition~\ref{prop:cdf-transform})]
    By definition, since $Z$ is a continuous random variable with full support, it admits a density $f_Z(z) > 0$ on $\mathbb{R}$, of which $F_Z$ is a primitive.
    Since (by the Fundamental Theorem of Calculus) $F_Z^\prime = f_Z > 0$, then $F_Z$ is strictly monotonic and thus invertible.
    Denote its inverse with $F_Z^{-1}:(0, 1)\to\mathbb{R}$.
    Now observe that, given $u\in(0, 1)$, we have
    \begin{equation*}
        F_U(u) = \mathbb{P}(U \le u) = \mathbb{P}(Z \le F_Z^{-1}(u)) = F_Z(F_Z^{-1}(u)) = u
    \end{equation*}
    Since $F_U$ is a cumulative distribution function, then it must be right continuous, which entails $F_U(u) = 0$ for $u\le0$, and
    monotonically increasing, which entails $F_U(u) \ge 1$ for $u\ge 1$, as well as having image $[0, 1]$ which entails $F_U(u) = 1$ for $u\ge 1$. $F_U$ is exactly the cumulative distribution function of the uniform-$(0, 1)$ distribution, which concludes the proof.
\end{proof}
\begin{proof}[Proof (of Corollary~\ref{cor:cdf-transform-compact})]
    By restricting the domain, we observe that $F_U$ is invertible on that restriction.
    By identical arguments the proof holds.
\end{proof}
\begin{proof}[Proof (of Lemma~\ref{lemma:cond-indep-of-u})]
    $u_1$ and $u_2$ are continuous (thus, Borel-measurable) transformations of conditionally independent random variables.
    The sigma-algebrae generated by measurable functions of the sigma-algebrae induced by $y_1 | \bm{x}_1, \bm{x}_2$, $\bm{\theta}$ and $y_1 | \bm{x}_1, \bm{x}_2$, $\bm{\theta}$ are sub-algebrae of the respective two.
    Thus, conditionally on $\bm{x}_1, \bm{x}_2, \bm{\theta}$, $u_1$ and $u_2$ are independent.
\end{proof}
\begin{proof}[Proof (of Lemma~\ref{lemma:distribution-of-min-uniforms})]
    Observe that, by the properties of the minimum,\linebreak$U^\prime \le u \Leftrightarrow 1 - 2\min\left\{U, 1-U\right\}  \le u \Leftrightarrow \frac{1-u}{2} \le \min\left\{U, 1-U\right\} \Leftrightarrow \big(\frac{1-u}{2} \le$ $\le U \big)\wedge \big(\frac{1-u}{2} \le 1-U\big)$.
    Therefore,
    \begin{equation*}
        \textstyle \mathbb{P}(U^\prime \le u) = \textstyle \mathbb{P}\left(\frac{1-u}{2} \le U, \frac{1-u}{2} \le 1 - U\right) =
        \textstyle \mathbb{P}\left(\frac{1-u}{2} \le U \le \frac{1+u}{2}\right)
        \stackrel{u\in(0,1)}{=} u,
    \end{equation*}
    which is precisely the density of a Uniform on $(0, 1)$.
\end{proof}
\begin{proof}[Proof (of Corollary~\ref{cor:distribution-AS})]
    Because of Corollary~\ref{cor:cdf-transform-compact}, $Q\left(\mathcal{P}_t^{k}; \bm{\theta}\right)$ is uniformly distributed.
    If we replace $U$ with $Q\left(\mathcal{P}_t^{k}; \bm{\theta}\right)$, by Lemma~\ref{lemma:distribution-of-min-uniforms}, $AS_{\bm{\theta}}\left(\mathcal{P}_t^{k}\right)$ is once again uniformly distributed on the interval $(0, 1)$, conditionally on $\bm{\theta}$.
\end{proof}

\section{Model Selection}\label{app:pareto-model-selection}
As mentioned,
to perform the Failure Detection task described in Section~\ref{sec:anomaly-scores},
we must carefully balance response fit and coverage, meaning that credible intervals should contain a fraction of points that is very close to the specified confidence.
We thus want to define a measure that, unlike~\eqref{eq:CIC}, morally considers \textit{any} level of confidence.
This is because the Anomaly Score defined in Section~\ref{sec:anomaly-scores} is reasonable under the assumption that  $F(\bullet|\bm{x}_t; \bm{\theta})$ closely matches the conditional law of the response, given $(\bm{x}_t, \bm{\theta})$.
Since it is impossible to evaluate all levels of confidence, we select finitely many values for the confidence $\alpha$, equispaced in the $(0, 1)$ interval, i.e., $\alpha \in \mathcal{K} = \{1/K, \dots, (K-1)/K\}$, for some positive integer $K$.
For each $\alpha$, we start by counting the number of Train observations falling within the posterior predictive credible intervals, i.e.,
\begin{equation}
	\textstyle C_{\alpha} = \sum_{t=1}^{N}\mathds{1}\left\{y_t\in\textup{CI}_{\alpha}\left(\bm{x}_t\right)\right\}
\end{equation}
with $\textup{CI}_{\alpha}\left(\bm{x}_t\right)$ computed empirically from a sample from the posterior predictive distribution evaluated at each $\bm{x}_t$.
Notice that this interval is different than the one appearing in~\eqref{eq:CIC}, as it does not depend on a single posterior sample.
If the posterior predictive density were to match exactly the data generating mechanism, then, each $C_{\alpha}$ should be distributed as a binomial with number of trials $N$ and success probability $\alpha$.
As a cost functional on the realization of counts $\{C_{\alpha}, \alpha \in \mathcal{K}\}$, we thus propose to:
\begin{itemize}
	\item Compute the negative log-probability of observing $\{C_{\alpha}$, assuming an underlying binomial with parameters $(N, \alpha)$;
	\item Summing up the resulting values for all $\alpha \in \mathcal{K}$.
\end{itemize}
Based on previous coverage functional and the PSIS-LOO performance index defined in Section~\ref{ssec:fit-metrics}, both computed on the Train set, we thus introduce the procedure for model selection, aimed at trading off between predictive fit and model coverage, described in Algorithm~\ref{alg:pareto-model-selection} (adapted from~\citet{2025raffaugoliniCoCoAFusEMixturesExperts}).
\begin{algorithm}[tb!]
	\caption{Model Selection Procedure Pseudocode}\label{alg:pareto-model-selection}
	\begin{algorithmic}
		\Require \textit{metric}$()$ (performance index functional)
		\Require \textit{coverage}$()$ (coverage functional)
		\Require $\nu$ (probability lower bound threshold for selection)
		\Require trials (experiments with different hyperparameters)
		\State $\textup{paretoSet} \gets \textup{pareto}(\textup{trials};\ \textup{\textit{metric}},\ \textup{\textit{coverage}})$
		\State $\textup{best} \gets \varnothing$
		\For{$p \textup{ \textbf{in} } \textup{sortAscending}(\textup{paretoSet};\ \textup{\textit{coverage}})$}
		\If{$\textup{best} = \varnothing$ \textbf{or} $\textup{ChebyshevLB}(\textup{\textit{metric}}(\textup{best}) < \textup{\textit{metric}}(\textup{p})) > \nu$}
			\State {$\textup{best} \gets p$}
		\EndIf
		\EndFor\\
		\Return best
	\end{algorithmic}
\end{algorithm}
The idea is to explore the Pareto frontier under increasing fit and decreasing coverage (in this order). For every new visited trial, we compute the Chebyshev lower bound on the probability that the current iterate has a higher metric than the previous best. If the Chebyshev lower bound exceeds a predetermined threshold $\nu \in (0, 1)$, then we accept the current trial as best. Otherwise, we discard it even when the computed metric is higher, until termination.
\begin{remark}
	The Chebyshev lower bound computation assumes model evaluations are independent, with mean and standard deviation equal to the associated performance index and its standard error, respectively. Under these assumptions, the Chebyshev lower bound is a \textit{pessimistic} estimate of the probability of improvement under configurations with worse coverage.
    In all our tests, we fix $\nu = 0.5$, i.e., we want to be convinced that, under our assumptions, we improve performance (at least) more often than not.
\end{remark}

%% file: references.bib
@article{1979barrowSplineNotationApplied,
  title = {Spline {{Notation Applied}} to a {{Volume Problem}}},
  author = {Barrow, D. L. and Smith, P. W.},
  year = {1979},
  journal = {The American Mathematical Monthly},
  volume = {86},
  number = {1},
  eprint = {2320304},
  eprinttype = {jstor},
  pages = {50--51},
  publisher = {[Taylor \& Francis, Ltd., Mathematical Association of America]},
  issn = {0002-9890},
  doi = {10.2307/2320304},
  urldate = {2025-04-09}
}

@book{1996raoHandbookConditionMonitoring,
  title = {Handbook of {{Condition Monitoring}}},
  author = {Rao, B. K. N.},
  year = {1996},
  publisher = {Elsevier},
  googlebooks = {ka\_rg1grHEUC},
  isbn = {978-1-85617-234-9},
  langid = {english}
}

@article{2005choyVibrationMonitoringDamage,
  title = {Vibration {{Monitoring}} and {{Damage Quantification}} of {{Faulty Ball Bearings}}},
  author = {Choy, F. K. and Zhou, J. and Braun, M. J. and Wang, L.},
  year = {2005},
  month = oct,
  journal = {Journal of Tribology},
  volume = {127},
  number = {4},
  pages = {776--783},
  issn = {0742-4787, 1528-8897},
  doi = {10.1115/1.2033899},
  urldate = {2024-12-10},
  langid = {english}
}

@article{2005samuelReviewVibrationbasedTechniques,
  title = {A Review of Vibration-Based Techniques for Helicopter Transmission Diagnostics},
  author = {Samuel, Paul D. and Pines, Darryll J.},
  year = {2005},
  month = apr,
  journal = {Journal of Sound and Vibration},
  volume = {282},
  number = {1-2},
  pages = {475--508},
  issn = {0022460X},
  doi = {10.1016/j.jsv.2004.02.058},
  urldate = {2024-12-09},
  copyright = {https://www.elsevier.com/tdm/userlicense/1.0/},
  langid = {english}
}

@article{2009chandolaAnomalyDetectionSurvey,
  title = {Anomaly Detection: {{A}} Survey},
  shorttitle = {Anomaly Detection},
  author = {Chandola, Varun and Banerjee, Arindam and Kumar, Vipin},
  year = {2009},
  month = jul,
  journal = {ACM Computing Surveys},
  volume = {41},
  number = {3},
  pages = {1--58},
  issn = {0360-0300, 1557-7341},
  doi = {10.1145/1541880.1541882},
  urldate = {2025-05-14},
  langid = {english}
}

@article{2012yukselTwentyYearsMixture,
  title = {Twenty {{Years}} of {{Mixture}} of {{Experts}}},
  author = {Yuksel, Seniha Esen and Wilson, Joseph N. and Gader, Paul D.},
  year = {2012},
  month = aug,
  journal = {IEEE Transactions on Neural Networks and Learning Systems},
  volume = {23},
  number = {8},
  pages = {1177--1193},
  issn = {2162-2388},
  doi = {10.1109/TNNLS.2012.2200299},
  urldate = {2023-11-15}
}

@article{2014leiConditionMonitoringFault,
  title = {Condition Monitoring and Fault Diagnosis of Planetary Gearboxes: {{A}} Review},
  shorttitle = {Condition Monitoring and Fault Diagnosis of Planetary Gearboxes},
  author = {Lei, Yaguo and Lin, Jing and Zuo, Ming J. and He, Zhengjia},
  year = {2014},
  month = feb,
  journal = {Measurement},
  volume = {48},
  pages = {292--305},
  issn = {02632241},
  doi = {10.1016/j.measurement.2013.11.012},
  urldate = {2024-12-10},
  langid = {english}
}

@article{2014zimrozDiagnosticsBearingsPresence,
  title = {Diagnostics of Bearings in Presence of Strong Operating Conditions Non-Stationarity---{{A}} Procedure of Load-Dependent Features Processing with Application to Wind Turbine Bearings},
  author = {Zimroz, Radoslaw and Bartelmus, Walter and Barszcz, Tomasz and Urbanek, Jacek},
  year = {2014},
  month = may,
  journal = {Mechanical Systems and Signal Processing},
  volume = {46},
  number = {1},
  pages = {16--27},
  issn = {08883270},
  doi = {10.1016/j.ymssp.2013.09.010},
  urldate = {2024-12-10},
  langid = {english}
}

@article{2015ambhoreToolConditionMonitoring,
  title = {Tool {{Condition Monitoring System}}: {{A Review}}},
  shorttitle = {Tool {{Condition Monitoring System}}},
  author = {Ambhore, Nitin and Kamble, Dinesh and Chinchanikar, Satish and Wayal, Vishal},
  year = {2015},
  journal = {Materials Today: Proceedings},
  volume = {2},
  number = {4-5},
  pages = {3419--3428},
  issn = {22147853},
  doi = {10.1016/j.matpr.2015.07.317},
  urldate = {2025-05-09},
  copyright = {https://www.elsevier.com/tdm/userlicense/1.0/},
  langid = {english}
}

@article{2016dasVibrationbasedDamageDetection,
  title = {Vibration-Based Damage Detection Techniques Used for Health Monitoring of Structures: A Review},
  shorttitle = {Vibration-Based Damage Detection Techniques Used for Health Monitoring of Structures},
  author = {Das, Swagato and Saha, P. and Patro, S. K.},
  year = {2016},
  month = jul,
  journal = {Journal of Civil Structural Health Monitoring},
  volume = {6},
  number = {3},
  pages = {477--507},
  issn = {2190-5452, 2190-5479},
  doi = {10.1007/s13349-016-0168-5},
  urldate = {2024-12-09},
  langid = {english}
}

@article{2017vehtariPracticalBayesianModel,
  title = {Practical {{Bayesian}} Model Evaluation Using Leave-One-out Cross-Validation and {{WAIC}}},
  author = {Vehtari, Aki and Gelman, Andrew and Gabry, Jonah},
  year = {2017},
  month = sep,
  journal = {Statistics and Computing},
  volume = {27},
  number = {5},
  pages = {1413--1432},
  issn = {1573-1375},
  doi = {10.1007/s11222-016-9696-4},
  urldate = {2024-01-29},
  langid = {english}
}

@article{2019stetcoMachineLearningMethods,
  title = {Machine Learning Methods for Wind Turbine Condition Monitoring: {{A}} Review},
  shorttitle = {Machine Learning Methods for Wind Turbine Condition Monitoring},
  author = {Stetco, Adrian and Dinmohammadi, Fateme and Zhao, Xingyu and Robu, Valentin and Flynn, David and Barnes, Mike and Keane, John and Nenadic, Goran},
  year = {2019},
  month = apr,
  journal = {Renewable Energy},
  volume = {133},
  pages = {620--635},
  issn = {09601481},
  doi = {10.1016/j.renene.2018.10.047},
  urldate = {2024-12-11},
  langid = {english}
}

@article{2019wangVibrationBasedCondition,
  title = {Vibration Based Condition Monitoring and Fault Diagnosis of Wind Turbine Planetary Gearbox: {{A}} Review},
  shorttitle = {Vibration Based Condition Monitoring and Fault Diagnosis of Wind Turbine Planetary Gearbox},
  author = {Wang, Tianyang and Han, Qinkai and Chu, Fulei and Feng, Zhipeng},
  year = {2019},
  month = jul,
  journal = {Mechanical Systems and Signal Processing},
  volume = {126},
  pages = {662--685},
  issn = {08883270},
  doi = {10.1016/j.ymssp.2019.02.051},
  urldate = {2024-12-09},
  langid = {english}
}

@article{2020mubaraaliIntelligentFaultDiagnosis,
  title = {Intelligent Fault Diagnosis in Microprocessor Systems for Vibration Analysis in Roller Bearings in Whirlpool Turbine Generators Real Time Processor Applications},
  author = {Mubaraali, L. and Kuppuswamy, N. and Muthukumar, R.},
  year = {2020},
  month = jul,
  journal = {Microprocessors and Microsystems},
  volume = {76},
  pages = {103079},
  issn = {01419331},
  doi = {10.1016/j.micpro.2020.103079},
  urldate = {2024-12-10},
  langid = {english}
}

@article{2021mauricioAdvancedSignalProcessing,
  title = {Advanced Signal Processing Techniques for Helicopter's Gearbox Monitoring},
  author = {Mauricio, A and Wang, W and Antoni, J and Gryllias, K},
  year = {2021},
  month = may,
  journal = {Journal of Physics: Conference Series},
  volume = {1909},
  number = {1},
  pages = {012043},
  issn = {1742-6588, 1742-6596},
  doi = {10.1088/1742-6596/1909/1/012043},
  urldate = {2024-12-10},
  langid = {english}
}

@article{2021soolonwahRegressionbasedDamageDetection,
  title = {A Regression-Based Damage Detection Method for Structures Subjected to Changing Environmental and Operational Conditions},
  author = {Soo Lon Wah, William and Chen, Yung-Tsang and Owen, John S},
  year = {2021},
  month = feb,
  journal = {Engineering Structures},
  volume = {228},
  pages = {111462},
  issn = {01410296},
  doi = {10.1016/j.engstruct.2020.111462},
  urldate = {2024-12-10},
  langid = {english}
}

@article{2021turnbullCombiningSCADAVibration,
  title = {Combining {{SCADA}} and Vibration Data into a Single Anomaly Detection Model to Predict Wind Turbine Component Failure},
  author = {Turnbull, Alan and Carroll, James and McDonald, Alasdair},
  year = {2021},
  month = mar,
  journal = {Wind Energy},
  volume = {24},
  number = {3},
  pages = {197--211},
  issn = {1095-4244, 1099-1824},
  doi = {10.1002/we.2567},
  urldate = {2024-12-10},
  langid = {english}
}

@article{2021vandeschootBayesianStatisticsModelling,
  title = {Bayesian Statistics and Modelling},
  author = {Van De Schoot, Rens and Depaoli, Sarah and King, Ruth and Kramer, Bianca and M{\"a}rtens, Kaspar and Tadesse, Mahlet G. and Vannucci, Marina and Gelman, Andrew and Veen, Duco and Willemsen, Joukje and Yau, Christopher},
  year = {2021},
  month = jan,
  journal = {Nature Reviews Methods Primers},
  volume = {1},
  number = {1},
  pages = {1--26},
  publisher = {Nature Publishing Group},
  issn = {2662-8449},
  doi = {10.1038/s43586-020-00001-2},
  urldate = {2025-01-17},
  copyright = {2021 Springer Nature Limited},
  langid = {english}
}

@article{2022badihiComprehensiveReviewSignalBased,
  title = {A {{Comprehensive Review}} on {{Signal-Based}} and {{Model-Based Condition Monitoring}} of {{Wind Turbines}}: {{Fault Diagnosis}} and {{Lifetime Prognosis}}},
  shorttitle = {A {{Comprehensive Review}} on {{Signal-Based}} and {{Model-Based Condition Monitoring}} of {{Wind Turbines}}},
  author = {Badihi, Hamed and Zhang, Youmin and Jiang, Bin and Pillay, Pragasen and Rakheja, Subhash},
  year = {2022},
  month = jun,
  journal = {Proceedings of the IEEE},
  volume = {110},
  number = {6},
  pages = {754--806},
  issn = {0018-9219, 1558-2256},
  doi = {10.1109/JPROC.2022.3171691},
  urldate = {2025-05-14},
  copyright = {https://creativecommons.org/licenses/by/4.0/legalcode},
  langid = {english}
}

@article{2022britoExplainableArtificialIntelligence,
  title = {An Explainable Artificial Intelligence Approach for Unsupervised Fault Detection and Diagnosis in Rotating Machinery},
  author = {Brito, Lucas C. and Susto, Gian Antonio and Brito, Jorge N. and Duarte, Marcus A.V.},
  year = {2022},
  month = jan,
  journal = {Mechanical Systems and Signal Processing},
  volume = {163},
  pages = {108105},
  issn = {08883270},
  doi = {10.1016/j.ymssp.2021.108105},
  urldate = {2024-12-09},
  langid = {english}
}

@article{2022huIntelligentAnomalyDetection,
  title = {An {{Intelligent Anomaly Detection Method}} for {{Rotating Machinery Based}} on {{Vibration Vectors}}},
  author = {Hu, Di and Zhang, Chen and Yang, Tao and Chen, Gang},
  year = {2022},
  month = jul,
  journal = {IEEE Sensors Journal},
  volume = {22},
  number = {14},
  pages = {14294--14305},
  issn = {1530-437X, 1558-1748, 2379-9153},
  doi = {10.1109/JSEN.2022.3179740},
  urldate = {2024-12-10},
  copyright = {https://ieeexplore.ieee.org/Xplorehelp/downloads/license-information/IEEE.html},
  langid = {english}
}

@article{2022leoniNewComprehensiveMonitoring,
  title = {A New Comprehensive Monitoring and Diagnostic Approach for Early Detection of Mechanical Degradation in Helicopter Transmission Systems},
  author = {Leoni, Jessica and Tanelli, Mara and Palman, Andrea},
  year = {2022},
  month = dec,
  journal = {Expert Systems with Applications},
  volume = {210},
  pages = {118412},
  issn = {09574174},
  doi = {10.1016/j.eswa.2022.118412},
  urldate = {2024-07-13},
  langid = {english}
}

@article{2022tiboniReviewVibrationBasedCondition,
  title = {A {{Review}} on {{Vibration-Based Condition Monitoring}} of {{Rotating Machinery}}},
  author = {Tiboni, Monica and Remino, Carlo and Bussola, Roberto and Amici, Cinzia},
  year = {2022},
  month = jan,
  journal = {Applied Sciences},
  volume = {12},
  number = {3},
  pages = {972},
  issn = {2076-3417},
  doi = {10.3390/app12030972},
  urldate = {2024-12-09},
  copyright = {https://creativecommons.org/licenses/by/4.0/},
  langid = {english}
}

@article{2023fengReviewVibrationbasedGear,
  title = {A Review of Vibration-Based Gear Wear Monitoring and Prediction Techniques},
  author = {Feng, Ke and Ji, J.C. and Ni, Qing and Beer, Michael},
  year = {2023},
  month = jan,
  journal = {Mechanical Systems and Signal Processing},
  volume = {182},
  pages = {109605},
  issn = {08883270},
  doi = {10.1016/j.ymssp.2022.109605},
  urldate = {2024-12-09},
  langid = {english}
}

@article{2023mengAnomalyDetectionConstruction,
  title = {Anomaly Detection for Construction Vibration Signals Using Unsupervised Deep Learning and Cloud Computing},
  author = {Meng, Qiuhan and Zhu, Songye},
  year = {2023},
  month = jan,
  journal = {Advanced Engineering Informatics},
  volume = {55},
  pages = {101907},
  issn = {14740346},
  doi = {10.1016/j.aei.2023.101907},
  urldate = {2024-12-10},
  langid = {english}
}

@article{2023romanssiniReviewVibrationMonitoring,
  title = {A {{Review}} on {{Vibration Monitoring Techniques}} for {{Predictive Maintenance}} of {{Rotating Machinery}}},
  author = {Romanssini, Marcelo and De Aguirre, Paulo C{\'e}sar C. and {Compassi-Severo}, Lucas and Girardi, Alessandro G.},
  year = {2023},
  month = jun,
  journal = {Eng},
  volume = {4},
  number = {3},
  pages = {1797--1817},
  issn = {2673-4117},
  doi = {10.3390/eng4030102},
  urldate = {2024-12-09},
  copyright = {https://creativecommons.org/licenses/by/4.0/},
  langid = {english}
}

@article{2023surucuConditionMonitoringUsing,
  title = {Condition {{Monitoring}} Using {{Machine Learning}}: {{A Review}} of {{Theory}}, {{Applications}}, and {{Recent Advances}}},
  shorttitle = {Condition {{Monitoring}} Using {{Machine Learning}}},
  author = {Surucu, Onur and Gadsden, Stephen Andrew and Yawney, John},
  year = {2023},
  month = jul,
  journal = {Expert Systems with Applications},
  volume = {221},
  pages = {119738},
  issn = {09574174},
  doi = {10.1016/j.eswa.2023.119738},
  urldate = {2025-05-14},
  langid = {english}
}

@misc{2023wadeBayesianDependentMixture,
  title = {Bayesian Dependent Mixture Models: {{A}} Predictive Comparison and Survey},
  shorttitle = {Bayesian Dependent Mixture Models},
  author = {Wade, Sara and Inacio, Vanda and Petrone, Sonia},
  year = {2023},
  month = jul,
  number = {arXiv:2307.16298},
  eprint = {2307.16298},
  primaryclass = {stat},
  publisher = {arXiv},
  doi = {10.1214/24-STS966},
  urldate = {2024-02-15},
  archiveprefix = {arXiv}
}

@article{2023yessoufouClassificationRegressionbasedConvolutional,
  title = {Classification and Regression-Based Convolutional Neural Network and Long Short-Term Memory Configuration for Bridge Damage Identification Using Long-Term Monitoring Vibration Data},
  author = {Yessoufou, Fadel and Zhu, Jinsong},
  year = {2023},
  month = nov,
  journal = {Structural Health Monitoring},
  volume = {22},
  number = {6},
  pages = {4027--4054},
  issn = {1475-9217, 1741-3168},
  doi = {10.1177/14759217231161811},
  urldate = {2024-12-10},
  langid = {english}
}

@article{2024arenaConceptualFrameworkMachine,
  title = {A Conceptual Framework for Machine Learning Algorithm Selection for Predictive Maintenance},
  author = {Arena, Simone and Florian, Eleonora and Sgarbossa, Fabio and S{\o}lvsberg, Endre and Zennaro, Ilenia},
  year = {2024},
  month = jul,
  journal = {Engineering Applications of Artificial Intelligence},
  volume = {133},
  pages = {108340},
  issn = {09521976},
  doi = {10.1016/j.engappai.2024.108340},
  urldate = {2025-03-27},
  langid = {english}
}

@article{2024jungAIBasedAnomalyDetection,
  title = {{{AI-Based Anomaly Detection Techniques}} for {{Structural Fault Diagnosis Using Low-Sampling-Rate Vibration Data}}},
  author = {Jung, Yub and Park, Eun-Gyo and Jeong, Seon-Ho and Kim, Jeong-Ho},
  year = {2024},
  month = jun,
  journal = {Aerospace},
  volume = {11},
  number = {7},
  pages = {509},
  issn = {2226-4310},
  doi = {10.3390/aerospace11070509},
  urldate = {2024-12-10},
  copyright = {https://creativecommons.org/licenses/by/4.0/},
  langid = {english}
}

@misc{2024leoniExplainableDatadrivenModeling,
  title = {Explainable Data-Driven Modeling via Mixture of Experts: Towards Effective Blending of Grey and Black-Box Models},
  shorttitle = {Explainable Data-Driven Modeling via Mixture of Experts},
  author = {Leoni, Jessica and Breschi, Valentina and Formentin, Simone and Tanelli, Mara},
  year = {2024},
  month = jan,
  number = {arXiv:2401.17118},
  eprint = {2401.17118},
  primaryclass = {cs, eess},
  publisher = {arXiv},
  urldate = {2024-07-24},
  archiveprefix = {arXiv},
  langid = {english}
}

@misc{2025raffaugoliniCoCoAFusEMixturesExperts,
  title = {{{CoCoAFusE}}: {{Beyond Mixtures}} of {{Experts}} via {{Model Fusion}}},
  shorttitle = {{{CoCoAFusE}}},
  author = {Raffa Ugolini, Aurelio and Tanelli, Mara and Breschi, Valentina},
  year = {2025},
  month = may,
  number = {arXiv:2505.01105},
  eprint = {2505.01105},
  primaryclass = {cs},
  publisher = {arXiv},
  doi = {10.48550/arXiv.2505.01105},
  urldate = {2025-05-05},
  archiveprefix = {arXiv}
}
